\documentclass[letterpaper,twocolumn,10pt]{article}
\usepackage{usenix-2020-09}
\usepackage[available,functional]{usenixbadges}

\usepackage{nopageno}
\usepackage{tikz}
\usepackage{float}
\usepackage{amsmath}
\usepackage{balance}
\usepackage{booktabs}
\usepackage[capitalise]{cleveref}
\usepackage{xcolor}
\usepackage{xspace}
\usepackage{multirow}
\usepackage{tabularx}
\usepackage{caption}
\usepackage{subcaption}
\usepackage{graphicx}
\usepackage{subcaption}
\usepackage{thm-restate}
\usepackage{tablefootnote}

\usepackage{enumitem}
\setitemize{noitemsep,topsep=2pt,parsep=2pt,partopsep=0pt}

\usepackage{amssymb}
\usepackage{amsthm}
\usepackage{mathtools}
\usepackage{subcaption}
\DeclareMathOperator*{\argmax}{arg\,max}
\DeclareMathOperator*{\argmin}{arg\,min}
\usepackage{algorithm}
\usepackage{algpseudocode}
\algtext*{EndWhile} %
\algtext*{EndIf} %
\algtext*{EndFor} %

\theoremstyle{plain}
\usepackage{thmtools}
\newtheorem{theorem}{Theorem}[section]

\newtheorem{lemma}[theorem]{Lemma}

\theoremstyle{definition}
\newtheorem{definition}[theorem]{Definition}

\theoremstyle{remark}
\newtheorem{remark}[theorem]{Remark}

\newif\ifcomments
\commentstrue 
\ifcomments

    \newcommand{\todo}[1]{{\color{orange}{\bf\sf [TODO: #1]}}}
    
\else
    \newcommand{\todo}[1]{}
    
\fi

\long\def\ignore#1{}

\def\Snospace~{\S{}}

\begin{document}

\date{}

\title{\Large \bf Fairness in Serving Large Language Models}

\author{
\rm{
Ying Sheng$^{\text{1,2}}$ \enskip
Shiyi Cao$^{\text{1}}$ \enskip
Dacheng Li$^{\text{1}}$ \enskip
Banghua Zhu$^{\text{1}}$ \enskip
Zhuohan Li$^{\text{1}}$ \enskip
Danyang Zhuo$^{\text{3}}$ \enskip
}
\\
\rm{
Joseph E. Gonzalez$^{\text{1}}$ \enskip
Ion Stoica$^{\text{1}}$ \enskip
}
\\
\\
$^{\text{1}}$UC Berkeley\enskip
$^{\text{2}}$Stanford University\enskip
$^{\text{3}}$Duke University\enskip
}

\maketitle

{\let\thefootnote\relax\footnote{{$^*$Part of the work was done when Ying
was visiting UC Berkeley.}}}

\begin{abstract}
High-demand LLM inference services (e.g., ChatGPT and BARD) support a wide range of requests from short chat conversations to long document reading.
To ensure that all client requests are processed fairly, most major LLM inference services have request rate limits, to ensure that no client can dominate the request queue. 
However, this rudimentary notion of fairness also results in under-utilization of the resources and poor client experience when there is spare capacity. 
While there is a rich literature on fair scheduling, serving LLMs presents new challenges due to their unpredictable request lengths and their unique batching characteristics on parallel accelerators. 
This paper introduces the definition of LLM serving fairness based on a cost function that accounts for the number of input and output tokens processed.
To achieve fairness in serving, we propose a novel scheduling algorithm, the Virtual Token Counter (VTC), a fair scheduler based on the continuous batching mechanism.
We prove a $2\times$ tight upper bound on the service difference between two backlogged clients, adhering to the requirement of work-conserving.
Through extensive experiments, we demonstrate the superior performance of VTC in ensuring fairness, especially in contrast to other baseline methods, which exhibit shortcomings under various conditions.
The reproducible code is available at \url{https://github.com/Ying1123/VTC-artifact}.
\end{abstract}
\section{Introduction}
\label{sec:introduction}

In a very short time, Large Language Models (LLMs), such as ChatGPT-4 Turbo~\cite{openai2023gpt4turbo}, have been integrated into various application domains, e.g., programming assistants, customer support, document search, and chatbots. The core functionality rendered by LLM providers to these applications is serving their requests. In addition to the response accuracy, the request response time is a key metric that determines the quality of service being provided. Furthermore, LLM providers seek to utilize their resources efficiently so they can reduce costs and increase their competitiveness in the market.

Today’s LLM serving systems~\cite{kwon2023efficient, tgi} typically use First-Come-First-Serve (FCFS) to schedule incoming requests. While simple, this scheduling discipline has several drawbacks. One such drawback is the lack of \emph{isolation}: a client sending a disproportionate number of requests can negatively impact the service of all the other clients sharing the same server (i.e.,  slow down their requests or even cause timeouts) even when they send very little traffic.
In multi-tenant personalized serving (S-LoRA~\cite{sheng2023slora}, Punica~\cite{punica}) that uses a dedicated adapter for each user, it is important to ensure fairness among the adapters as well.
One solution to address this problem is to limit the incoming load of each client. Many of the existing LLM services do this today by imposing a request-per-minute (RPM) limit~\cite{openai2023rate} for each client.

Unfortunately, RPM can lead to low resource utilization. A client sending requests at a high rate will be restricted even if the system is underutilized. This leads to wasted resources, an undesirable situation given the cost and the scarcity of GPUs. Thus, we want a solution that provides not only isolation (like RPM limit) but also high resource utilization.

This is a common problem in many other domains like networking and operating systems. The solution of choice to achieve both isolation and high resource utilization in those domains has been \emph{fair queueing}~\cite{fq}. Fair queueing ensures that each client will get their “fair share”. In the simplest case, if there are $n$ clients sharing the same resource, the fair share is at least $1/n$ of the resource, which means that each client gets at least $1/n$ of the resource. Furthermore, if some clients do not use their share, other clients with more demands can use it, hence leading to higher resource utilization. 

In this paper, we apply fair sharing to the domain of LLM serving at the token granularity. We do it at the token rather than request granularity to avoid unfairness due to request heterogeneity. Consider two clients, client $A$ sends requests of 2K tokens each (both input and output), and client $B$ sends requests of 200 tokens each. Serving an equal number of requests for each client would be unfair to client $B$ as her requests consume much fewer resources than client $A$’s requests. This is similar to networking where fair queuing is typically applied to the bit granularity, rather than packet granularity.

Despite these similarities, we cannot directly use the algorithms developed for networking and operating systems, as LLM serving has several unique characteristics. First, the request output lengths are unknown in advance. In contrast, in networking, the packet lengths are known before the packet is scheduled. 
Second, the cost of each token can vary. For instance, the cost of processing an input (prompt) token is typically lower than that of an output token, because input token processing is parallelizable. %
In contrast, the cost of sending a bit or the cost of a CPU time slice are the same irrespective of the workload.
Third, the \emph{effective} capacity of an LLM server (i.e., processing rate expressed in token/sec when the request queue is non-empty) can vary over time.
For example, 
longer input sequences take more memory. This limits the number of batched parallel requests during generation, leading to GPU under-utilization and a lower processing rate.
In contrast, the network or CPU capacity is assumed to be fixed.

In this paper, we discuss the factors that need to be considered when defining fairness in the context of LLM serving.
We show how different definitions can be incorporated into a configurable service cost function in Section~\ref{sec:definition}.
While the cost function can be customized, a simple metric of counting input and output tokens at different prices is extensively used in analysis for the sake of simplicity.
We then present a fair scheduling algorithm called \textit{Virtual Token Counter (VTC)} that can be easily adapted for different service cost functions. At a high level, VTC tracks the services received for each client and will prioritize the ones with the least services received, with a counter lift each time a client is new to the queue.
It updates the counters at a token-level granularity on the fly, which addresses the unknown length issue.
VTC integrates seamlessly with current LLM serving batching techniques (Section~\ref{sec:prelim_llm_serving}), and its scheduling mechanism does not depend on the server's capacity, overcoming the problem of the dynamically fluctuating server capacity.
We also provide theoretical bounds of fairness for VTC in~\Cref{sec:vtc}. The serving architecture of VTC is illustrated in Figure~\ref{fig:architecture}.

In summary, this paper makes the following contributions:
\begin{itemize}
    \item This is the first work to discuss the fair serving of Large Language Models to the best of our knowledge.
    We identify its unique challenges and give the definition of LLM serving fairness (\Cref{sec:definition}).
    \item We propose a simple yet effective fair-serving algorithm called VTC. We provide rigorous proofs for VTC on fairness guarantee, which gives fairness bound within $2\times$ of the optimal bound (\Cref{sec:method}).
    \item We conduct in-depth evaluations on our proposed algorithm VTC. Results confirm that our proposed algorithms are fair and work-conserving
    (\Cref{sec:eval-service}).
\end{itemize}

\section{Background}

In this section, we first introduce how an LLM serving system operates.
Then we describe existing methods for ensuring fairness in LLM serving.

\subsection{Large Language Models Serving}

\label{sec:prelim_llm_serving}

\paragraph{LLM serving with a single request}
First, a request contains information about its arrival time ($a$), input tokens ($x$), and its associated client ($u$). Formally, we represent a request using a three-tuple ($a$, $x$, $u$). The system generates output tokens based on the input tokens. For instance, the input tokens can be an incomplete sentence, and the system generates the rest of the sentence~\cite{openai_api}.

The generation procedure consists of two stages: the initial~\textbf{prefilling} stage, and the~\textbf{decoding} stage~\cite{pope2023efficiently}.
Mathematically, $x$ is a sequence of tokens $(x_1, x_2, ..., x_n)$. In the prefilling stage, the LLM computes the probability of the first new tokens: $P(x_{n+1} | x_1, ..., x_n)$. In the decoding stage, the system~\textit{autoregressively} generates a new token. At time $t$ ($t \geq 1$), the process is written as: $P(x_{n+t+1} | x_1, ..., x_{n+t})$.

The decoding stage ends when the LLM generates a special end-of-sentence (EOS) token or the number of generated tokens reaches a pre-defined maximal length.

\begin{figure}[t]
{\includegraphics[width=0.45\textwidth]{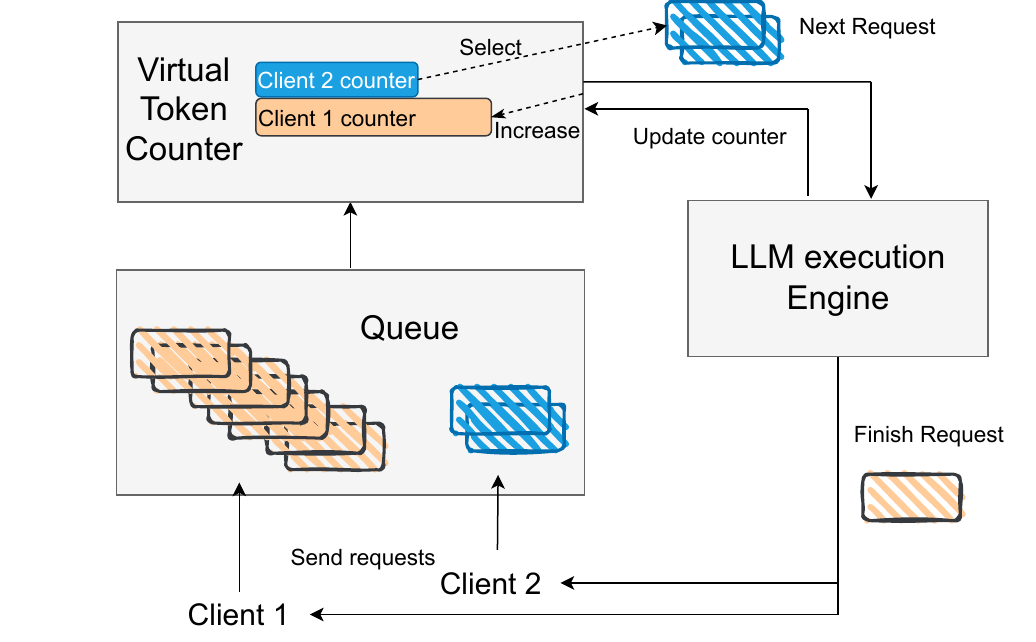}}
    \caption{Serving architecture with Virtual Token Counter (VTC), illustrated with two clients. VTC maintains a queue of requests and keeps track of tokens served for each client.
    In each iteration of the LLM execution engine, some tokens from some clients are generated. The counters of these clients are correspondingly updated.
    When the condition of adding new requests is satisfied (e.g. memory is released when some other requests finish), VTC will be invoked to choose the requests to be added.
    VTC achieves fairness by prioritizing clients with the lowest counter and carefully handling clients' leave and rejoin (Section~\ref{sec:vtc}).
    \label{fig:architecture}
    }
\end{figure}

\paragraph{LLM serving with multiple requests}
In the online serving scenarios, multiple clients submit requests to the serving system. To process these requests, the system maintains two concurrent streams: A monitoring stream adds requests to a waiting queue; an execution stream selects and executes request(s) from the waiting queue.

Naively, the execution stream can choose to execute requests one by one. However, this is highly GPU inefficient due to various natures of the LLM generation procedure. For instance, the decoding steps must be carried out sequentially where the arithmetic intensity is relatively low in a single step. Contemporary serving systems usually perform batching that executes multiple requests concurrently to maximize the system throughput.
The most widely used approach in LLM serving is continuous batching~\cite{yu2022orca}. Algorithm~\ref{alg:high_level} shows the pseudocode for continuous batching.\footnote{For a simple presentation, we consider an implementation that only uses continuous batching for decode steps but keeps the prefill step separated. as how TGI~\cite{huggingface2023tgi} adopted the original proposed iteration-level scheduling in Orca~\cite{yu2022orca}. For more discussions, see \Cref{sec:vtc_integration}.}
The monitoring stream enqueues requests to a waiting queue. The execution stream performs a check on whether there are finished requests at the end of each decoding step. If there are, the system removes these requests and adds new requests from the queue.
\paragraph{Fairness with continuous batching}
We can naturally integrate fairness policies into the continuous batching algorithm, by 
designing a fair \texttt{select\_new\_requests()} function in Algorithm~\ref{alg:high_level}.
Intuitively, the execution stream should keep track of how much service a particular client has received, and prioritize clients that haven't received much service in the next selection. 
We formally define fairness in the LLM serving context in Section~\ref{sec:definition} and design a method with theoretical guarantee in Section~\ref{sec:method}. 

We adopt a continuous batching scheme in which a request only leaves the batch when it generates an EOS token or reaches the pre-defined maximum number of generated tokens (i.e., no preemption). This paper focuses on integrating fair scheduling with continuous batching, and we leave an investigation on preemption as an orthogonal future work (discussed in \Cref{sec:preemption}).

\begin{algorithm}[ht]
\caption{LLM serving with Continuous batching}
\begin{algorithmic}[1]
\State Initialize current batch $B \leftarrow \emptyset$, waiting queue $Q \leftarrow \emptyset$
\State $\triangleright$ \texttt{with monitoring stream:}
\While{True}
    \If {new request $r$ arrived}
        \State $Q \leftarrow Q + r$
    \EndIf
\EndWhile
\State $\triangleright$ \texttt{with execution stream:}
\While{True}
    \If{can\_add\_new\_request()}
        \State $B_{new} \leftarrow {\textbf{select\_new\_requests(Q)}}$
        \State prefill($B_{new}$)
        \State $B \leftarrow B + B_{new}$
    \EndIf
    \State decode($B$)
    \State $B \leftarrow$ filter\_finished\_requests($B$)
\EndWhile
\end{algorithmic}
\label{alg:high_level}
\end{algorithm}

\subsection{Existing Fairness Approaches}
\label{sec:prelim_fair_sharing}

Fairness is a key metric of interest in computer systems that provide service to multiple concurrent clients~\cite{bertsekas2021data}.
A \textit{fair} LLM serving system should protect clients from a misbehaving client who may try to overload the serving system by submitting too many requests.

\paragraph{RPM Limit Per Client}
As a common practice of API management (e.x. \cite{openai2023rate}), specific rate limits are established for each client's API usage to prevent potential abuse or misuse of the API and ensure equitable access for all clients. This limitation is on the metric request-per-minute (RPM). Once a client reaches the RPM limit, the client is only allowed to submit more requests in the next time window. However, it's important to note that while these limits are effective in managing resource allocation during periods of high demand, they may not be \textit{work-conserving} when the number of active clients is low. In such scenarios, the system's capacity might be underutilized, as the imposed limits prevent the full exploitation of available resources. 

\paragraph{Fair Queueing~\cite{fq}} The fairness problem has been extensively studied in the past for traditional compute resources, such as CPU cycles and network bandwidth. Fair queuing and its variants (e.g., Weighted Fair Queuing (WFQ)~\cite{wfq}, Self-clocked Fair Queueing~\cite{self-clock}, and Start-time Fair Queueing (SFQ)~\cite{sfq}) have been proposed to achieve the fair allocation of link bandwidth in packet-switching networks.

In the traditional packet-switching network, a \emph{flow} $f$ is referred to as a sequence of packets $p_f^0, p_f^1, \ldots p_f^n$ transmitted by a source.
Each packet $p_f^j$ is of length $l_f^j$. A flow is \emph{backlogged} during the time interval $[t_1,t_2)$ if it has one or more outstanding packets waiting in the queue at any time $t\in[t_1,t_2)$.

All fair queueing algorithms maintain a system \emph{virtual time}, $v(t)$, which intuitively measures the service received by a continuously backlogged flow in terms of bits forwarded. Each packet, $p$ is associated two tags: \textit{Start} tag $S(p)$ and a \textit{Finish} tag $F(p) = S(p) + l_p$. The Start tag (a.k.a. packet's virtual starting time) is computed based on both the system virtual time and the Finish tag (a.k.a. packet's virtual finishing time). These algorithms schedule packets in the ascending order of either the Finish or Start tags.

In networking, fairness is simply defined as follows: for any two flows, $f$ and $g$, that are backlogged  during time interval $[t_1,t_2)$, we have
\begin{equation}
\left|W_f(t_1,t_2)-W_g(t_1, t_2)\right| \leq U(f,g),
\label{eq:fq_fairness}
\end{equation}
where $W_f(t_1,t_2)$ and $W_g(t_1,t_2)$ denote the service received in bits by flow $f$ and $g$, respectively, during interval $[t_1,t_2)$, and $U(f,g)$ is a function of the properties of flows $f$ and $g$ (e.g.,  maximum packet length) and the system (e.g., link capacity). 
Intuitively, for packets-switching networks, the allocation of a link bandwidth is fair if, for any time interval during which two flows are backlogged, each of these flows receives approximately the same service in terms of the number of bits being forwarded during that interval. A scheduling algorithm is said to be \emph{work-conserving} if a link always forwards packets when the queue is not  empty~\cite{sfq_d}.

There exists a distinct strand of research~\cite{grrr,cfq,ion96} focusing on the fair scheduling of preemptible tasks (e.g., CPU scheduling). The Completely Fair Scheduler (CFS)~\cite{cfs}, implemented in Linux 2.6.23 and applying fair queuing to CPU scheduling, is closely related to our algorithm. In CFS, a ``vruntime'' is maintained for each task, and the task with the smallest ``vruntime'' is scheduled next. The tasks can be presented with a small time slice, aiming to maximize overall CPU utilization while also maximizing interactive performance.

\subsection{Challenges}
\label{sec:challenge}
There are several unique challenges in LLM serving that prevent a direct application of fair-queuing-like algorithms. The first challenge is that the definition of fairness in the context of LLM serving is unexplored, and likely very different than that discussed in fair-queuing literature.

Traditional fairness is defined by measuring the cost of requests, which is usually a fixed value that is easy to estimate in either network or operating systems.
For example, in networking, requests correspond to packets, and the cost is usually the number of bits of a packet. However, in LLM generations, how to define the cost of a request is not obvious. The cost per token can vary. Especially, processing an input (prompt) token is typically less expensive than processing an output token, as input tokens are processed in parallel while output tokens must be generated sequentially. 
Batching the output tokens from different requests can parallelize the fully connected layers but is still slower than processing input tokens for the attention layers.

Additionally, in LLM serving, the server has variable token-rate capacity, although the memory allocated for a batch is constant.
Firstly, even if the request queue is not empty, we are not guaranteed that each batch is full.
This is because we need to preserve spaces for future generated tokens, and also because the tokens added to the batch are not at the token but the request granularity. 
Secondly, the number of tokens processed highly depends on the requests' arrival patterns because of the continuous batching mechanism (Section~\ref{sec:prelim_llm_serving}).
Furthermore, the capacity depends on the mix between input and output tokens of existing requests. If all requests have long past tokens, then the capacity is likely to be low (See~\Cref{fig:var_capacity}).
Then there is no way to define a fixed amount of equal share.

\begin{figure}[ht!]
    \centering
    \includegraphics[width=0.475\textwidth]{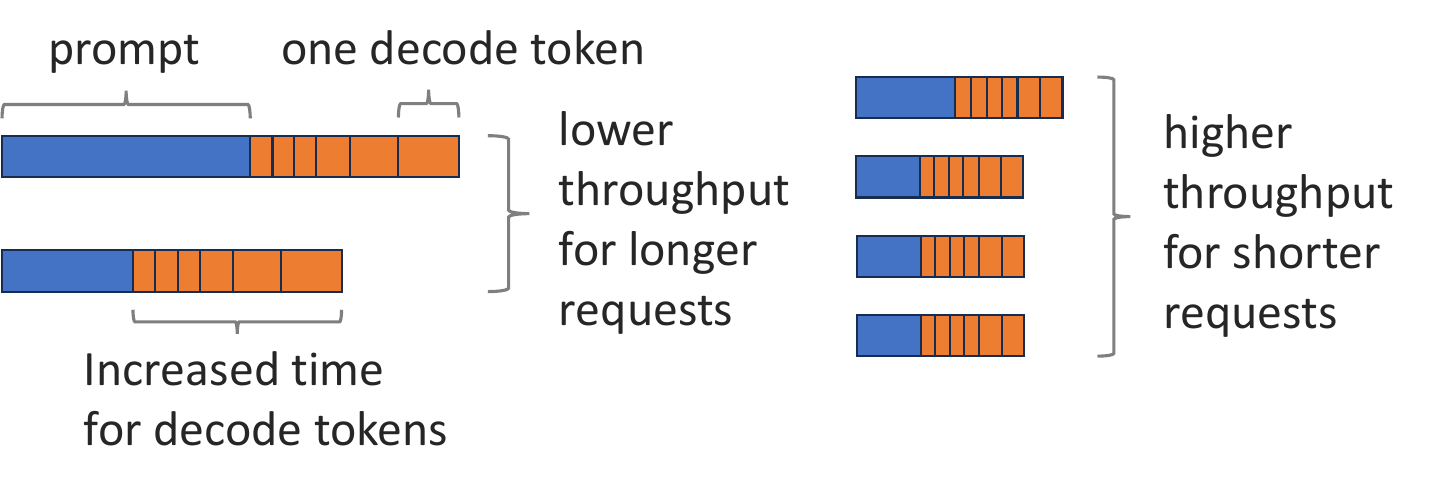}
    \caption{
    An illustration of how request length can affect the cost and server capacity in terms of throughput.
    The visualized length is not precise but for illustration purposes only.}
    \label{fig:var_capacity}
\end{figure}

The second challenge is the characteristic of unknown output length before finishing a request.
This prevents a direct adaptation of classical algorithms like SFQ and Deficit Round Robin (DRR)~\cite{drr} into the LLM serving.
SFQ-style algorithms can provide good bounds in fairness by setting the Start and Finish tags through virtual time, as introduced in \Cref{sec:prelim_fair_sharing}.
However, computing Start and Finish tags requires knowing the request length in advance.
DRR performs round-robin scheduling with a ``deficit counter'' mechanism to achieve fair scheduling of packets of variable length. 
In DRR, each client is assigned a specific quantum of service. It tracks the ``deficit'' of service for each client to ensure fairness over time. During each round, the scheduler allows each client to dispatch as many requests as possible, provided that the total length of these requests does not exceed the sum of the client's assigned quantum for that round and any accumulated deficit from previous rounds.
Without knowing the length in advance, DRR cannot determine how many jobs can be scheduled within the quantum.
Compared to CPU scheduling, although exploring adequate preemption is worthwhile in LLM serving, it cannot occur frequently. We need to define service fairness in LLM serving and operate at the granularity of individual tokens when frequent preemption is not possible. Additionally, the Completely Fair Scheduler (CFS) in CPU scheduling does not account for the concurrency of each task. It seeks fairness among individual tasks rather than among streams of tasks that can be executed concurrently.
\footnote{Our algorithm does not consider preemption. Discussion about preemption is in \Cref{sec:preemption}.}

We will give the definition for LLM serving fairness in \Cref{sec:definition} and give a scheduling algorithm to achieve the LLM serving fairness in \Cref{sec:method}.
We will outline our algorithm in a basic format for clarity, while details on its general form and integration with existing serving frameworks can be found in \Cref{sec:vtc_integration}. Although our algorithm is closely related to CFS, we also discuss the adaptation of DRR in \Cref{sec:drr}. Further discussions on future work are included in \Cref{sec:future_work}.

\section{Definition of Fairness in LLM Serving}
\label{sec:definition}

In this section, we discuss the cost of a request, and the measurement of the service a client has received (\Cref{sec:measurement}).
After defining the measurement of service, we can define fairness among clients in \Cref{sec:def_fairness}.

\begin{table}[ht]
\centering
\resizebox{1\columnwidth}{!}{
\begin{tabular}{c|p{70mm}}
\toprule
Notation & Explanation\\
\midrule
$W_f(t_1, t_2)$ & service received by $f$ during interval $[t_1, t_2)$ (write as $W(t_1, t_2)$ when $f$ is clear in the context) \\
$n_p$ & number of processed input tokens \\
$n_q$ & number of processed output tokens \\
$w_p$ & weight of input tokens in the cost function \\
$w_q$ & weight of output tokens in the cost function \\
$h(n_p, n_q)$ & customized cost function \\
\midrule
$c_i$ & virtual token counter for client $i$ \\
$Q$ & waiting queue of requests to be processed \\
$i \in Q$ & $\exists r \in Q$, $r$ is a request from client $i$ \\
$L_{input}$ & maximum number of input tokens in a request \\
$L_{output}$ & maximum number of output tokens in a request \\
$M$ & maximum number of tokens that can be fitted in a running batch  \\
$U$ & invariant bound: $\max (w_p \cdot L_{input}, w_q\cdot M)$ \\
\bottomrule
\end{tabular}
}
\caption{The upper half includes notations for service measurement. The lower half includes notations for the VTC algorithm and its analysis.
The terms $n_p, n_q$ can refer to either a single request or a single client, depending on the context.}
\label{tab:notations}
\end{table}

\subsection{Measurement of Service}
\label{sec:measurement}

In this subsection, we discuss the measurement of the service a client has received.
Specifically, we define $W_f(t_1, t_2)$ and $W_g(t_1, t_2)$ from \Cref{eq:fq_fairness} in the context of LLM serving. We omit the subscript and write $W(t_1, t_2)$ when the client is clear from the context or is irrelevant.
The number of processed input and output tokens are denoted as $n_p, n_q$.
Notations that will be introduced and used multiple times in this paper are summarized in \Cref{tab:notations}.

\textbf{Number of tokens}
A straightforward way to measure the service provided to a client is by summing the number of input tokens that have been processed and the number of output tokens that have been generated so far, i.e., $W(t_1, t_2) = n_p(t_1, t_2) + n_q(t_1, t_2)$ during the time window $[t_1, t_2)$.

\textbf{Number of FLOPs} Alternatively, one can measure the total FLOPs used in each stage, i.e., $W(t_1, t_2) = \mathit{FLOP}_{\mathit{input}}(t_1, t_2) + \mathit{FLOP}_\mathit{output}(t_1, t_2)$. This can be more precise because it captures the difference among tokens in attention computation, where tokens with longer prefixes require more computation.

However, both of these formulations cannot accurately reflect the actual LLM serving cost: The computation of the tokens at the prefill stage can be parallelized and achieve high GPU utilization. However, at the generation stage, we can only generate tokens one by one, as each token depends on all previous tokens as described in~\Cref{sec:prelim_llm_serving}. %

\textbf{Weighted number of tokens} %
To better reflect the actual LLM serving cost, a more accurate measure should capture the difference in costs of the prefilling and generation phases. One simple way to implement this idea is by using a weighted combination of the prefilling (input) tokens and decoding (output) tokens, inspired by the pricing mechanism used in OpenAI's API\footnote{\url{https://openai.com/pricing}}.
Formally, let $w_p$ be the weight of input tokens and $w_q$ be the weight of output tokens. Then, we have $W(t_1, t_2) = w_p \cdot n_p(t_1, t_2) + w_q \cdot n_q(t_1, t_2)$.
Due to its simplicity, we will use this measure extensively in our analysis and evaluation.

\textbf{Customized, unified representation.}
The definition of fairness in LLM serving can also be extended to other aspects, such as the weighted number of FLOPs or a more sophisticated method introduced in \cite{narayanan2023cheaply} that uses piecewise linear functions for the number of input and output tokens.
Generally, the service can be represented as a function of the number of input and output tokens ($n_p, n_q$, respectively).
Let $h(n_p, n_q)$ be the cost function that is monotonically increasing according to $n_p$ and $n_q$.
Our method can easily accommodate different $h$ (\Cref{sec:general_cost}).

\subsection{Fairness in LLM Serving}
\label{sec:def_fairness}

In this paper, we apply fair sharing to the domain of LLM serving to provide performance isolation across multiple clients sharing the same LLM server. In particular, we employ the classic formulation of max-min fairness~\cite{bertsekas2021data}, which computes a \emph{fair share} for the clients sharing a given server. In a nutshell, given the metric of service fairness, if a client sends requests
at no more than its fair share, all its requests are served. In contrast, if a client sends requests at more than its fair share, its excess requests will be delayed or even dropped. As a result, a misbehaving client cannot deny the service to other clients, no matter how many requests it sends. 
To achieve max-min fairness,
an idealized serving system follows desirable properties as below:

\ignore{
The fair share of clients is computed as follows:
\begin{itemize}
    \item[1.] Initially, it tries to allocate shares equally among all clients.
    \item[2.] If not all the share can be allocated equally (due to some clients requiring less than their equal share), the excess resources are reallocated to those who need more. This process is repeated iteratively.
\end{itemize}
}

\ignore{
Intuitively, the least advantaged clients or tasks are given as much resource as possible, which prevents resource monopolization by a few high-capacity clients and promotes equitable resource distribution.
Now, we can define the three desired properties for a fair LLM serving.
}

\begin{enumerate}
    \item \textbf{Backlogged clients} Any two clients $f, g$ that are continuously  backlogged during a given time interval $[t_1, t_2)$ should receive the same service during this interval, i.e. $W_f(t_1, t_2) = W_g(t_1, t_2)$. 

    \item \textbf{Non-backlogged clients} Client $f$ that is continuously backlogged during time interval $[t_1, t_2)$ should not receive less service than another client, $g$, that is not continuously backlogged during the same time interval, i.e., $W_f(t_1, t_2) \geq W_g(t_1, t_2)$.

    \item \textbf{Work-conservation} As long as there are requests in the queue, the server should not be idle.
\end{enumerate}

\ignore{
\begin{itemize}
    \item[1.] \textbf{Backlogged clients} For any two clients $f, g$ that are both backlogged during a time period $[t_1, t_2)$, they should receive the same level of service, i.e. $W_f(t_1, t_2) = W_g(t_1, t_2)$.
    \item[2.] \textbf{Non-backlogged clients} If during time period $[t_1, t_2)$ a client sends requests at less than its fair share, the client should have all its requests serviced.
    \item[3.] \textbf{Work-conserving} We also want a scheduling algorithm that is work-conserving, i.e., as long as there are requests in the queue, the scheduler tries to keep the server busy.
\end{itemize}
}

The first property means that two clients sending requests at more than their fair share will get the same service, regardless of the discrepancy between their sending rates.
The second property says that a client sending requests at a higher rate will not get less service than a client sending at a lower rate. 
Basically, the first two properties say that a misbehaving client is contained (i.e., doesn't receive more service than other backlogged clients), and not punished (i.e., doesn't receive less service than other non-backlogged clients).
Finally, the work conserving property aims to maximize the utilization, addressing a key weakness of the RPM-based solutions.

The three properties above assume an idealized fair serving system. A practical system will approximate these properties. In general, the best we can achieve is deriving bounds that are independent of the length of the time interval, e.g., in the first property, the difference between $W_f(t_1, t_2)$ and $W_g(t_1, t_2)$ is bounded by a value that is independent of $t_2 - t_1$. 
We give the formal guarantees provided by our method in \Cref{sec:vtc}.

\ignore{
The three properties above are for a perfectly fair scheduling method.
When designing the algorithms, the three properties will be achieved approximately.
For example, the difference in the services received for two backlogged clients can be bounded by a value independent of $t_2 - t_1$.
The formal guarantees with our method are provided in \Cref{sec:vtc}.
}

\section{Achieving Fairness}
\label{sec:method}
In this section, we present our algorithm VTC with proved fairness properties in \Cref{sec:vtc}, and show its generalization for customized service measurement in \Cref{sec:general_cost}.
Variants of VTC, including weighted VTC and VTC with length prediction, are introduced in \Cref{sec:method_weighted_vtc} and \Cref{sec:method_vtc_length}.

\subsection{Virtual Token Counter (VTC)}
\label{sec:vtc}

Based on insights from prior discussions, we've identified key challenges inherent in large language model (LLM) serving that hinder direct adaptation of existing algorithms to deliver approximately fair LLM service.
We then propose the Virtual Token Counter (VTC), a mechanism for achieving fair sharing in LLM Serving (\Cref{alg:vtc}).
To quantify the service received by a client we use the weighted number of tokens metric, as described in \Cref{sec:measurement}.
We discuss the generalization to other metrics in \Cref{sec:general_cost}.

Intuitively, VTC tracks the services received for each client and will prioritize the ones with the least services received, with a counter lift each time a client is new to the queue.
The \emph{counter lift} is needed to fill the gap created by a low load period of the client, so that it will not be unfairly served more in the future.
In other words, the credits for a client are utilized immediately and cannot be carried over or accumulated.
The virtual counters are updated each time a new token is generated, which can reflect the services received instantly.
This operates at the token-level granularity, and thus addresses the unknown length issue.
VTC can be easily integrated into the continuous batching mechanism, and its scheduling mechanism does not depend on the server's capacity, overcoming the problem of variable token-rate capacity.

\begin{algorithm}[ht]
\caption{Virtual Token Counter (VTC)}
\begin{algorithmic}[1]
\Require request trace, input token weight $w_p$, output token weight $w_q$, upper bound from \Cref{eq:invariant} denoted as $U$.
\State let current batch $B \leftarrow \emptyset$
\State let $c_i \leftarrow 0$ for all client $i$
\State let $Q$ denote the waiting queue, which is dynamically changing.
\State $\triangleright$ \texttt{with monitoring stream:}
\While{True}
    \If {new request $r$ from client $u$ arrived}
        \If {not $\exists r' \in Q, client(r')=u$}
            \If {$Q=\emptyset$}
                \State let $l\leftarrow$ the last client left $Q$
                \State $c_u \leftarrow \max\{c_u, c_{l}\}$
            \Else
                \State $P \leftarrow \{i \mid \exists r' \in Q, client(r')=i\}$
                \State $c_u \leftarrow \max\{c_u, \min\{c_i \mid i\in P\}\}$
            \EndIf
        \EndIf
        \State $Q \leftarrow Q + r$
    \EndIf
\EndWhile
\State $\triangleright$ \texttt{with execution stream:}
\While{True}
    \If{can\_add\_new\_request()}
        \State $B_{new} \leftarrow \emptyset$
        \While{True} 
            \State let $k \leftarrow \argmin_{i \in \{client(r) \mid r\in Q\}} c_i$
            \State let $r$ be the earliest request in $Q$ from $k$.
            \If{$r$ cannot fit in the memory}
                \State Break
            \EndIf
            \State $c_k \leftarrow c_k + w_p\cdot input\_length(r)$
            \State $B_{new} \leftarrow B_{new} + r$
            \State $Q \leftarrow Q - r$
        \EndWhile
        \State forward\_prefill($B_{new}$)
        \State $B \leftarrow B + B_{new}$
    \EndIf
    \State forward\_decode($B$)
    \State $c_i\leftarrow c_i + w_q \cdot |\{r \mid client(r)=i, r \in B\}|$
    \State $B \leftarrow$ filter\_finished\_requests($B$)
\EndWhile
\end{algorithmic}
\label{alg:vtc}
\end{algorithm}

\Cref{alg:vtc} shows how VTC can be implemented in the continuous batching framework described in \Cref{sec:prelim_llm_serving}.
A more general integration for VTC is described in \Cref{sec:vtc_integration}.
It maintains a virtual counter for each client, denoted as $\{c_i\}$.
The counters are initialized as $0$ (line 2).
The program runs with two parallel streams.

The monitoring stream listens to the incoming requests, described in lines 5-14. The new request will be added to the waiting queue $Q$ immediately.
If the new request is the only request in $Q$ for its sender client, a counter lift (lines 8-13) will happen.
Because this client could have been underloaded before, its counter could be smaller than the other active counters.
However, since the credits cannot be carried over, we need to lift it to the same level as other active counters, thus maintaining fairness among this client and others.
Lines 9-10 address the scenario where the entire system was in an idle state. We do not reset all the counters to avoid nullifying a previously accumulated deficit upon a system restart.

The execution stream is the control loop of an execution engine that implements continuous batching.
Line 17 controls the frequency of adding a minibatch $B_{new}$ of new requests into the running batch $B$.
Commonly, the server will add a new minibatch after several decoding steps.
The minibatch $B_{new}$ is constructed by iteratively selecting the request from the client with the smallest virtual counter (lines 20-26). The counters will be updated when adding new requests according to the service invoked by the input tokens (line 24).
After each decoding step (line 29), $\{c_i\}$ will be updated immediately according to the service invoked by the newly generated output tokens (line 30).

The VTC algorithm is (mostly) work-conserving because it only manipulates the dispatch order and does not reject a request if it can fit in the batch.

\subsubsection{Fairness for backlogged clients in VTC}
\label{sec:proof-overload}

In this subsection, we provide the theoretical guarantee for fairness among overloaded clients in VTC.
More precisely, the overload of a client is reflected by its backlog, which can be formally defined as follows. Intuitively, a client being backlogged means its requests are queued up.

\begin{definition}[Backlog]
A client $f$ is backlogged during time interval $[t_1, t_2)$, if at any time $t \in [t_1, t_2)$, $f$ has a request that is waiting in the queue. 
\end{definition}

We adapt the traditional definition of fairness for backlogged clients in the network to our scenario.
The following definition formally defined the item 1 introduced in \Cref{sec:def_fairness}, that for any interval, and any two continuously backlogged clients during the time interval, the difference of their received service should be bounded by a value that is independent of the interval length.

\begin{definition}[Fairness adapted from \cite{stf}]
\label{def:fairness}
Let $W_f(t_1, t_2)$ be the aggregated service received by client $f$ in the interval $[t_1, t_2)$.
A schedule is fair w.r.t. $\delta$, if for any clients $f$ and $g$, for all intervals $[t_1, t_2)$ in which clients $f$ and $g$ are backlogged, we have
$ \vert W_f(t_1, t_2) - W_g(t_1, t_2) \vert \leq \delta $.
\end{definition}

In the rest of the paper, as in \Cref{alg:vtc}, we let $Q$ denote the set of requests in the waiting queue.
We abuse the notation of $i\in Q$ for a client $i$ to indicate there exists $r\in Q$, such that $r$ is a request from client $i$.
Let $L_{input}$ and $L_{output}$ be the maximum number of input and output tokens in a request.
Let $M$ be the maximum number of tokens that can be fitted in a running batch.
Lemma~\ref{lem:invariant} reflects the core design of \Cref{alg:vtc}, that the virtual counters for active clients are chasing each other to ensure their maximum difference is bounded.
The missing proof for Lemma~\ref{lem:invariant} and all following theorems are in \Cref{sec:proof}.

\begin{restatable}{lemma}{invariant}
\label{lem:invariant}
The following invariant holds at any time in \Cref{alg:vtc} when $Q\neq \emptyset$: \begin{equation}
\max_{i\in Q} c_i - \min_{i \in Q} c_i \leq \max (w_p \cdot L_{input}, w_q\cdot M)
\label{eq:invariant}
\end{equation}
\end{restatable}

We then introduce our main theorem which provides a bound for Definition~\ref{def:fairness}.
\begin{theorem}[Fairness for overloaded clients]
\label{thm:main_fairness}
For any clients $f$ and $g$, for any time interval $[t_1, t_2)$ in which $f$ and $g$ are backlogged, \Cref{alg:vtc} guarantees
\footnote{The service of a served request incurred by pre-filling (service for input tokens) is counted at the time when the request is added to the running batch (line 24 in \Cref{alg:vtc}), rather than the time when prefill is finished.
This is because we want to count the input tokens immediately to avoid selecting all the same $k$ at line 20 in \Cref{alg:vtc} for $B_{new}$.}
\[
\vert W_f(t_1, t_2) - W_g(t_1, t_2) \vert \leq  2\max(w_p\cdot L_{input}, w_q\cdot M).
\]
\end{theorem}

\begin{proof}
For any $f$, if $f$ is backlogged during time $t_1$ to $t_2$, we have $W_f(t_1, t_2) = c_f^{(t_2)} - c_f^{(t_1)}$.
This is because the line 7 will not be reached for client $f$ during $t_1$ to $t_2$, and the $c_f$ keeps increasing during $t_1$ to $t_2$ by adding $w_p$ product the number of served input tokens and $w_q$ product the number of served output tokens.
By Lemma~\ref{lem:invariant}, from \Cref{eq:invariant}, we have
\begin{align*}
\vert W_f(t_1, t_2) - W_g(t_1, t_2) \vert &\leq |c_f^{(t_1)} - c_g^{(t_1)}| + |c_f^{(t_2)} - c_g^{(t_2)}| \\
&\leq 2\max (w_p \cdot L_{input}, w_q\cdot M) 
\end{align*}
\end{proof}

\begin{remark}
    An empirical illustration of this theorem can be found in Figure~\ref{fig:syn_overload_diff}, where the difference between services received by backlogged clients is bounded, regardless of how long they have been backlogged.
\end{remark}
\begin{remark}
    Line 13 can be modified to take any value between $\min\{c_i | \exists r' \in Q, client(r') = i\}$ and $\max\{c_i | \exists r' \in Q, client(r') = i\}$. The proof of \Cref{thm:main_fairness} should still hold.
\end{remark}

\begin{remark}
\label{remark:restrict_mem}
    To tighten the bound in \Cref{thm:main_fairness}, we can restrict the memory usage for each client in the running batch.
    This might compromise the work-conserving property, as we will demonstrate in \Cref{thm:lower_bound}. Therefore, there is a \textit{trade-off} between achieving a better fairness bound and maintaining work conservation.
    Heuristically, predicting the request length in advance could result in a smaller discrepancy, as detailed in \Cref{sec:method_vtc_length}. Additionally, preemption is another method to achieve smaller differences, discussed in \Cref{sec:preemption}.
\end{remark}

We also prove in the next theorem that the bound in \Cref{thm:main_fairness} is tight within a factor of $2$ for a family of work-conserving schedulers.
We say a scheduler is work-conserving if it stops adding requests to a partially-filled minibatch (line 22 in \Cref{alg:vtc}) only when it runs out of memory\footnote{Different implementation may have different criteria of ``not enough memory''. This can only be achieved heuristically because the number of output tokens is unknown before it finishes.}
but not for fairness reasons.

\begin{restatable}{theorem}{lowerbound}
\label{thm:lower_bound}
For any work-conserving schedule without preemption, there exists some query arrival sequence such that for client $f, g$ and a time period $t_1, t_2$, such that
\[ \vert W_f(t_1, t_2) - W_g(t_1, t_2) \vert \geq w_q \cdot M, \]
where clients $f, g$ are backlogged during the time $[t_1, t_2)$.
\end{restatable}

As we mentioned before, output tokens are more expensive than input tokens, so normally we have $w_q > w_p$. Therefore the right-hand side of the inequality in \Cref{thm:main_fairness} is $2w_q\cdot M$, which is $2\times$ of the lower bound in \Cref{thm:lower_bound}.

\subsubsection{Fairness for non-backlogged clients in VTC}
\label{sec:proof-non-overload}

In this subsection, we discuss item 2 in \Cref{sec:def_fairness}.
A backlogged client will not receive less service than another client.
This can be reflected in the following theorem.

\begin{restatable}{theorem}{nopunish}
\label{thm:nopunish}
If a client $f$ is backlogged during time interval $[t_1, t_2)$, for any client $g$, there is
\[ W_f(t_1, t_2) \geq W_g(t_1, t_2) - 4 U. \]
Here $U$ is the upper bound from \Cref{eq:invariant}.
\end{restatable}

In addition to that,
clients who send requests constantly less than their share should have their requests serviced nearly instantly.
This property intuitively can be implied by the first item in \Cref{sec:def_fairness}, as if a low-rate client cannot be served on time, it becomes backlogged, which requires the same level of service with backlogged clients.
We formally prove this property to offer a fairness assurance for clients who are not overloaded. This intuitively acts as a safeguard against misbehaving clients~\cite{demers1989analysis}.

We start with Definition~\ref{def:capacity_bound} and \Cref{lem:simple_iso} discussing the aspect of latency bounds. Intuitively, if a client is not backlogged and has no requests running, the next request from it will be processed within a latency bound that is independent of the request rate of other clients.
\begin{definition}
\label{def:capacity_bound}
    Assume there are $n$ active clients during [$t_1$, $t_2$), and the server capacity at time $t \in [t_1, t_2)$ is defined as $S(t)$, where \[ \int_{t_1}^{t_2} S(t) \,dt = \sum_{i=1}^{n} W_{i}(t_1, t_2)\] 
    Because the server capacity is always positive and bounded, there exists $a, b\in \mathbf{R}^{+}$ such that $\forall$ t, $a < S(t) \leq b$.
\end{definition}

\begin{restatable}{theorem}{simpleiso}
    \label{lem:simple_iso}
    Let $A(r)$ and $D(r)$ denote the arrival time and dispatch time of a request $r$. Assume there are in total $n$ clients, $\forall t_1, t_2$, if at $t_1$, a client f is not backlogged and has no requests in the running batch, then the next request $r_f$ with $t_1 < A(r_f) < t_2$ will have its response time bounded: 
    \begin{equation}
        D(r_f) - A(r_f) \leq 2 \cdot (n-1) \cdot \frac{\max(w_p \cdot L_{input}, w_q\cdot M)}{a} 
    \end{equation}
    Here $a$ is the lower bound of the capacity in Definition~\ref{def:capacity_bound}.
\end{restatable}

\begin{remark}
    The bound in~\Cref{lem:simple_iso} is irrelevant to the request rate of others, giving an upper bound for latency against ill-behavior clients.    
\end{remark}

The above is about one request not getting delayed.
The following theorem shows that during time period $[t_1, t_2)$, if there are $n$ active clients sending requests, and client $f$ is sending requests with a rate constantly less than $1/n$ of the server's capacity (with some constant gap), client $f$ should have all its requests been served.

\begin{restatable}{theorem}{isolation}
(Fairness for non-overloaded clients)
\label{thm:isolation}
For any time interval $[t_1, t_2)$, we claim the following.

Assume a client $f$ is not backlogged at time $t_1$ and
for any time interval $[t, t_2), t_1\leq t < t_2$, $f$ has requested services less than $\frac{T(t, t_2)}{n(t, t_2)} - 5U$, where $T(t, t_2)$ is the total services received for all clients during the interval $[t, t_2)$, $n(t, t_2)$ is the number of clients that have requested services during the interval, and $U$ is the upper bound from \Cref{eq:invariant}.

Then, all of the services requested from $f$ during the interval $[t_1, t_2)$ will be dispatched.
\end{restatable}

\subsection{Adapt to Different Fairness Criteria}
\label{sec:general_cost}

\Cref{alg:vtc} is designed for fairness with the service function $W(t_1, t_2)$ as a linear combination of the number of processed input tokens and the number of generated tokens.
For a different definition of $W(t_1, t_2)$, \Cref{alg:vtc} can be easily modified to update the counter according to the other definitions described in \Cref{sec:measurement}.

Assume we aim for fairness using $\sum_{r} h(n_p^r, n_q^r)$ as the metric of service, where $h$ is a specific function. In this context, $r$ indexes the served requests, and $n_p^r, n_q^r$ represent the number of input and output tokens served for request $r$, respectively. 
Line 24 will be changed to
\[ c_k \leftarrow c_k + h(n_p^r, 0). \]
Line 30 will be changed to
\[ c_i \leftarrow c_i + \sum_{r\mid client(r) = i, r\in B} \left( h(n_p^r, n_q^r) - h(n_p^r, n_q^r - 1) \right). \]

The fairness bound will also be changed according to $h(\cdot, \cdot)$. Under the assumption that output tokens are more expensive than input tokens, the bound will become the maximum value of aggregated $h(\cdot, \cdot)$ for a set of requests that can be fitted in one running batch.
\Cref{alg:vtc_general} in \Cref{sec:vtc_integration} is the pseudocode of a general VTC framework.

\subsection{Weighted VTC}
\label{sec:method_weighted_vtc}
VTC can be applied when clients have tiers. Similar to weighted fair queuing, clients can have different weights to represent their priority in service. If a client \( f \) has a weight \( w_1 \), that is twice the weight \( w_2 \) of client \( g \), client \( f \) is expected to receive twice the service than client \( g \). When they are continuously backlogged during the interval \( [t_1, t_2) \), we want \(\left| \frac{W_f(t_1, t_2)}{w_1} - \frac{W_g(t_1, t_2)}{w_2} \right|\) to be bounded instead of \(\left| W_f(t_1, t_2) - W_g(t_1, t_2) \right|\).

Weighted VTC can be easily implemented by modifying the lines that update the virtual tokens.
For example, the line 22 in \Cref{alg:vtc_general} will be changed to
\[ c_i \leftarrow c_i + \dfrac{\sum_{r\mid client(r) = i} \left( h(n_p^r, n_q^r) - h(n_p^{r(old)}, n_q^{r(old)}) \right)}{w_i}. \]
Here $c_i$ is the virtual counter of client $i$, and $w_i$ is its corresponding weight.

\subsection{VTC with Length Prediction}
\label{sec:method_vtc_length}

As mentioned in Remark~\ref{remark:restrict_mem}, using VTC with length prediction can heuristically reduce the service discrepancy. In standard VTC, the counters only reflect served tokens. Tokens generated in the future can only be passively added to the counter. This results in a large service discrepancy because requests are overly added due to underestimation of their costs at the time of prompting, leading to the forced serving of over-compensated output tokens. Incorporating a prediction mechanism can help reduce this variance.

The theoretical worst-case scenario won’t change, according to the lower bound proved in \Cref{thm:lower_bound}.
But practically, the average-case service discrepancy could be smaller.

The modified pseudocode of VTC with length prediction is described in \Cref{alg:vtc_length_predict} in \Cref{sec:vtc_length}.
Intuitively, when a request \( r \) is selected, the cost associated with the predicted number of output tokens is immediately added to the virtual counter of the client sending the request. During the actual decoding process, adjustments are made to the virtual counter based on the actual number of output tokens produced. If the actual number of tokens exceeds the prediction, the virtual counter is increased accordingly. Conversely, if fewer tokens are generated than predicted when finished, the virtual counter is reduced. The effectiveness of the length predictor is contingent upon both the workload and the accuracy of predictions, as demonstrated in our evaluations.

\section{Evaluations}
\label{sec:eval-service}

In this section, we evaluate VTC against other alternatives under different workloads.
The results confirm the fairness properties introduced in \Cref{sec:definition} of VTC, and show that all other alternatives will fail in at least one workload.

\subsection{Setup}
\label{sec:setup}

\textbf{Implementation}
We implement our VTC and other baseline schedulers in S-LoRA\cite{sheng2023slora}, a system that serves a large amount of LoRA adapters concurrently.
Its backbone is a general serving system adapted from LightLLM~\cite{lightllm}.
It includes the implementation of continuous batching~\cite{yu2022orca} and PagedAttention~\cite{kwon2023efficient}\footnote{with block size equals 1.}.
Our VTC scheduler is built on top of those two techniques.
Our implementation is elegant and can be implemented as a thin layer on top of the existing scheduler, it contains only about 100 lines of code on top of S-LoRA. The simplicity demonstrates its wide applicability.
Fairness can be considered among general clients, and our experiments are done in this way.
But we would like to note that fairness also could be taken into consideration among adapters, especially under the scenario of personalization that uses one adapter per customer, which originally motivated this paper.

\paragraph{Baselines}
In this section, we benchmark VTC and the baselines as below:
\begin{itemize}
    \item First Come First Serve (FCFS):
    In the First-Come-First-Serve method, requests are handled strictly in the order they are received, irrespective of the requesting client.
    This is the default scheduling strategy in many prevalent LLM serving systems, including vLLM~\cite{kwon2023efficient} and Huggingface TGI~\cite{tgi}.

    \item Request per minute (RPM):
    This method limits the maximum number of requests that a client can make to the server within a one-minute timeframe. The definition of service corresponds to \Cref{sec:definition}.
    When a client exceeds this limit, subsequent requests are blocked until the limit resets at the start of the next minute.

    \item Least Counter First (LCF):
    This is a variant of VTC without the counter lift component.
    Each client will maintain a counter for the service it received so far.
    The request from the client with the smallest counter will be scheduled each time.
\end{itemize}
We also benchmark the VTC with length predictions as described below:
\begin{itemize}
    \item VTC (predict): This variant of VTC, detailed in \Cref{alg:vtc_length_predict}, utilizes the average output length of the last five requests from each client to predict the output length.

    \item VTC (oracle): This variant employs a hypothetical output length predictor that achieves 100\% accuracy.
\end{itemize}

\paragraph{Synthetic Workload}
We run Llama-2-7b on A10G (24GB), using the memory pool for the KV cache with size 10000\footnote{There are in total 10000 tokens for KV cache that can be stored on GPU.}.
We use various workloads to demonstrate different aspects of fairness, and compare VTC with other baselines.
The detailed results are in \Cref{sec:exp_synthetic}.
We start with synthetic workloads to give a clear message for fairness properties.

\paragraph{Real Workload}
To validate the effectiveness of VTC in more complex real-world scenarios, we also experiment with VTC and other baselines under workloads constructed from the trace log of LMSYS Chatbot Arena~\cite{zheng2023judging, zheng2023lmsys}, which is an LLM serving platform for real-world clients.

\paragraph{Ablation Study}
In the ablation study, to evaluate the impact of different memory pool sizes and request lengths on the scheduling fairness, we run Llama-2-13b on A100 (80GB) with a memory pool of size 35000 and 65000 respectively. For each memory pool size, we evaluate the absolute difference in the accumulated service of two clients.

\paragraph{Metrics}
We apply the weighted number of tokens described in \Cref{sec:measurement} as the measurement of services in our evaluation.
Following OpenAI pricing, we set $w_p = 1$ and $w_q = 2$.
\begin{itemize}
    \item The service received by client $i$ at time $t$ is measured as $W_i(t-T, t+T)$ for a certain $T$.
    \item The absolute difference in service between clients is quantified based on accumulated services, represented as $\max_{i, j} |W_i(0, t) - W_j(0, t)|$.
    \item The response time of client $i$ at time $t$ is measured as the average first token latency of the requests sent by client $i$ during the time window $[t-T, t+T]$.
\end{itemize}
In all settings, we set $T=30$ seconds.

We employ \emph{service difference} as a quantitative metric to assess the deviation from ideal fairness. A smaller difference in service indicates more equitable scheduling.
Formally, the service difference between two clients is defined as the minimum of two values: the difference between their received services, and the difference between the lower service and its corresponding request rate. For example, consider two clients that received services \(s_1\) and \(s_2\), such that \(s_1 \leq s_2\), and let \(r_1\) denote the request rate sent from the first client. Then the service difference is defined as \(\min(s_2 - s_1, \vert r_1 - s_1\vert )\).

\paragraph{VTC Variants}
The experiments for weighted VTC are presented in \Cref{sec:weighted_vtc}, demonstrating its capability to serve clients with varying priorities.
To illustrate the versatility of the service function beyond the linear model used in our primary analysis, we evaluate a profiled service cost function in \Cref{sec:vtc_profile}, which is a quadratic function.
Additional experiments on VTC with length prediction are detailed in \Cref{sec:vtc_length}.

\subsection{Results on Synthetic Workloads}
\label{sec:exp_synthetic}

We design a set of experiments to visualize the fairness properties of VTC. We start with synthetic traces to show plots reflecting the ideal case's fairness. We experiment from the simplest setting, where clients send requests following a uniform distribution with the same input and output length, to complex settings, where requests arrive stochastically, with various input and output lengths.

\begin{figure}[ht]
    \subfloat[Absolute difference for accumulated service]{\includegraphics[width=0.2\textwidth]{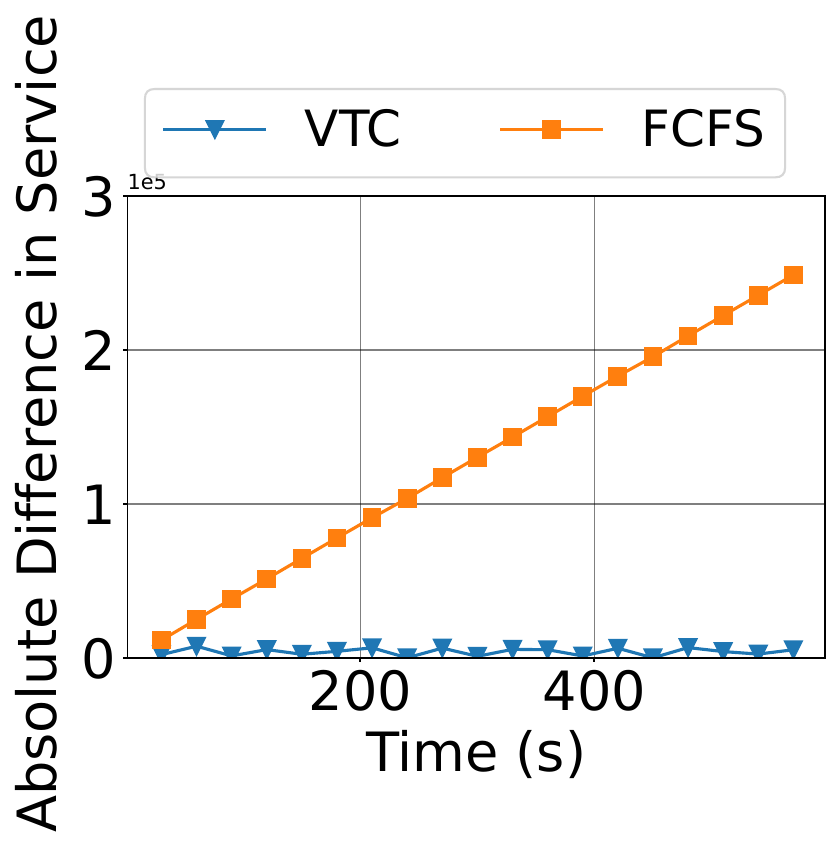}\label{fig:syn_overload_diff}}
    \hfill
    \subfloat[Received service rate, calculated as an average of 60s time windows. (VTC)]
    {\includegraphics[width=0.24\textwidth]{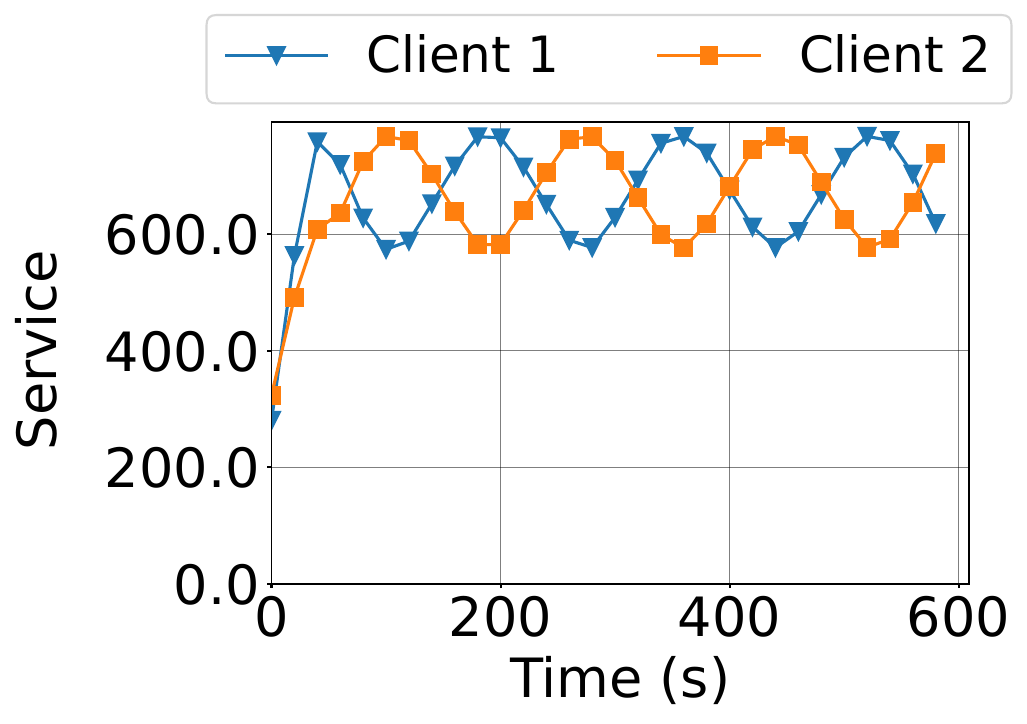}\label{fig:syn_overload_service}}
    \caption{Two clients with different request rates and both overloaded. Client 1 sends 90 requests per minute. Client 2 sends 180 requests per minute, both evenly spaced out so that each request is sent at a consistent time interval throughout the minute. Every request has input lengths of 256 and output lengths of 256. Both clients are backlogged because they exceed the server capacity.}
    \label{fig:syn_overload}
\end{figure}

\begin{figure}[t]
    \subfloat[Received service rate (VTC).]{\includegraphics[width=0.23\textwidth]{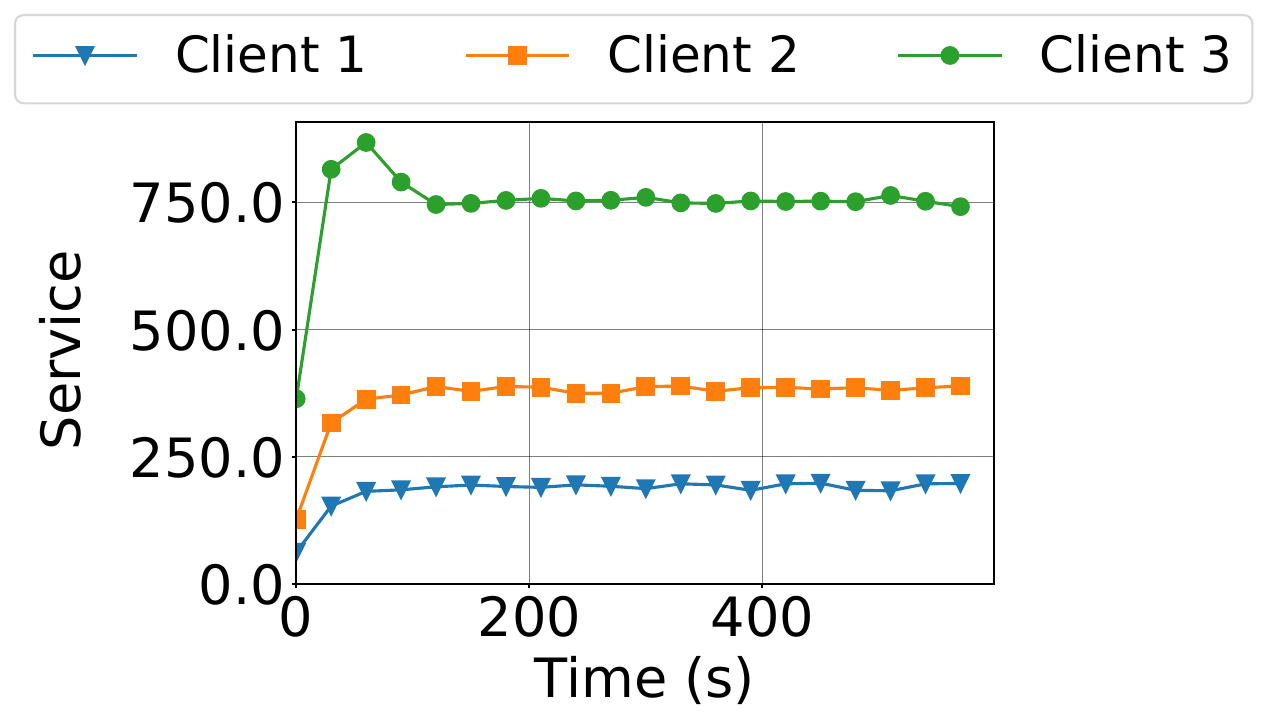}\label{fig:syn_proportional_service}}
    \hfill
    \subfloat[Response time (VTC).]{\includegraphics[width=0.23\textwidth]{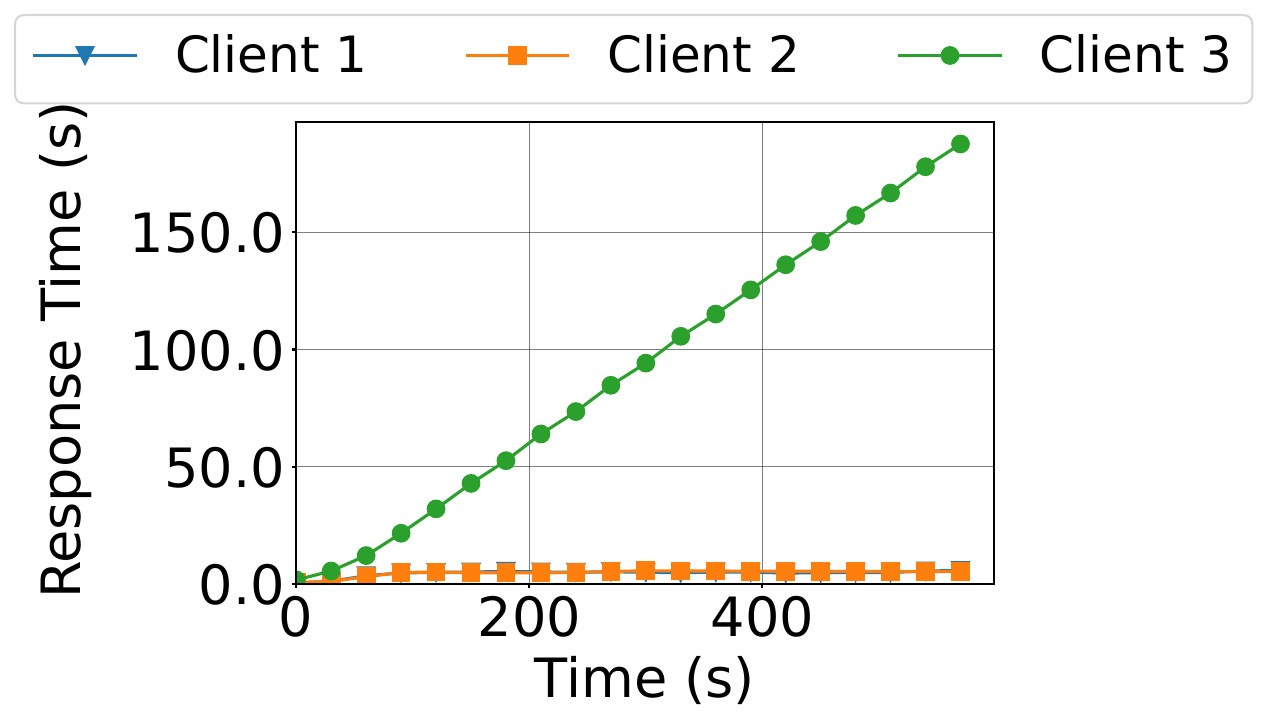}\label{fig:syn_proportional_rt}}
    \caption{
    Client 3 who is overloaded can consume more than its share as Clients 1 and 2 are sending requests lower than their share.
    Clients 1, 2, and 3 send 15, 30, and 90 requests per minute, respectively, under uniform distribution. Requests have input lengths of 256 and output lengths of 256. Client 3 is backlogged, while Clients 1 and 2 are not.}
    \label{fig:syn_proportional}
\end{figure}

\paragraph{Constant request rates}
We start with scenarios where requests arrive deterministically with the same input and output length.
In \Cref{fig:syn_overload}, two clients send requests at different rates, but are both constantly overloaded.
In this case, \Cref{fig:syn_overload_diff} shows VTC can keep the difference between services received by both clients to be small. 
FCFS cannot maintain fairness, which always serves more for the client who is sending requests at a higher rate.
\Cref{fig:syn_overload_service} shows the real-time received service rate for two clients in VTC, which confirms that the two received the same level of services at any time interval.
This experiment empirically validates Theorem~\ref{thm:main_fairness}.

In~\Cref{fig:syn_proportional}, three clients send requests at around $2/13$, $4/13$, and $>7/13$ of the server's capacity, respectively.
In this case, Clients 1 and 2 can be served immediately when their requests arrive (\Cref{fig:syn_proportional_rt}), and Client 3 will consume the remaining capacity (more than 1/3), which is an empirical illustration of the work-conserving property of VTC. 
The service received for Client 1 and Client 2 have a ratio 1 : 2, which is consistent with their request rates (15 versus 30).

\paragraph{ON/OFF request pattern}
In real-world applications, clients usually do not always send requests to the server. They may occasionally be idle (``OFF'' phase). We call this the ``ON/OFF'' pattern. In \Cref{fig:syn_on_off_less}, Client 2 is always in the ``ON'' phase, sending requests at a rate of 120 per minute. Client 1 sends 30 requests per minute (less than half of the capacity) during the ON phase and switches to OFF phase periodically. Since Client 1 uses less than half the system capacity when it is in the ON phase, its requests are mostly processed before it switches to the OFF phase (\Cref{fig:syn_on_off_less_rt}). When it is in the OFF phase, Client 1 thus takes all the system capacity. The total service rate remains the same, which confirms VTC's flexibility in achieving work-conserving.

On the contrary, in~\Cref{fig:syn_on_off_overload}, client 1 sends much more than half the capacity during the ON phase, and makes itself always backlogged. Thus, even when it is in the OFF phase, it is still in the backlog status. In this case, Client 1 and Client 2 should still receive the same level of service rate.

\begin{figure}[ht]
    \subfloat[Received service rate (VTC).]{\includegraphics[width=0.23\textwidth]{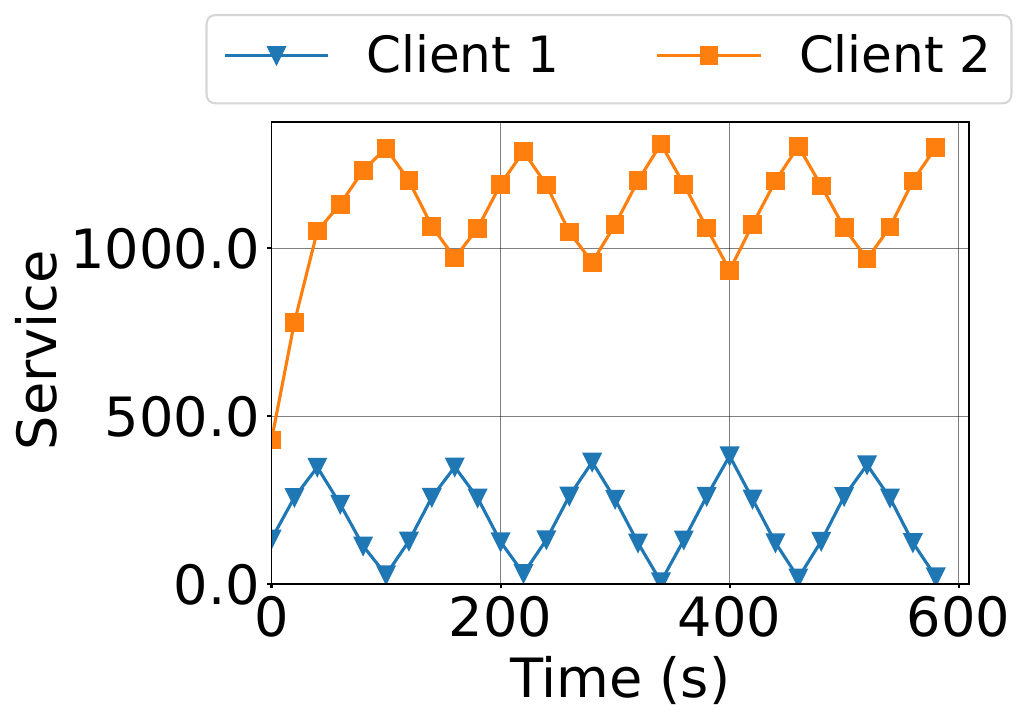}\label{fig:syn_on_off_less_service}}
    \hfill
    \subfloat[Response time (VTC).]{\includegraphics[width=0.23\textwidth]{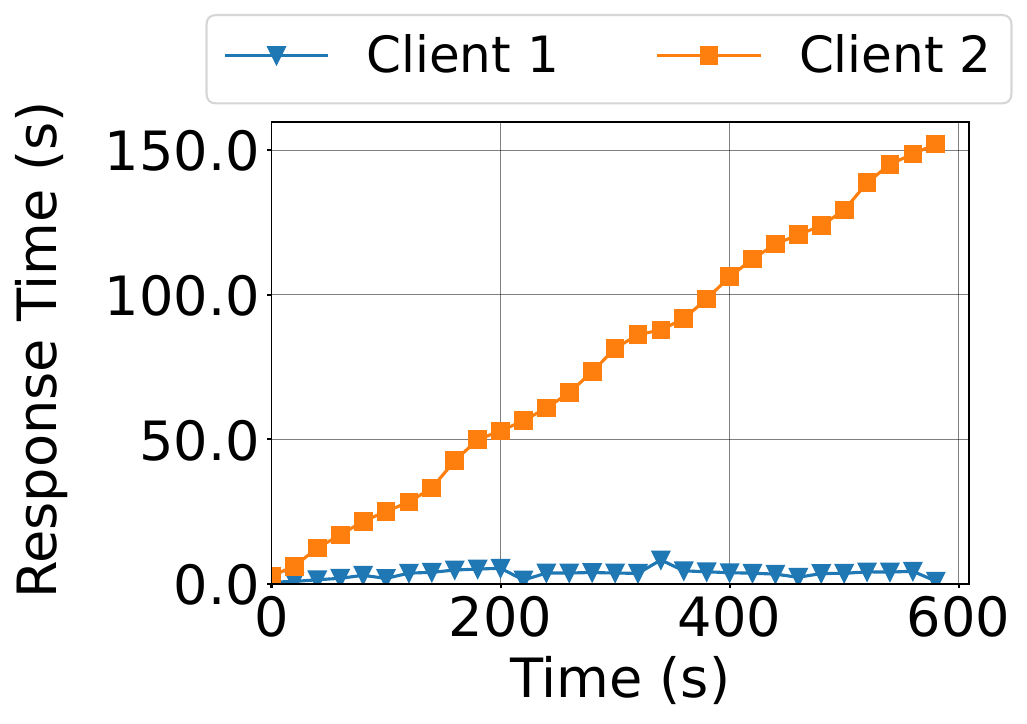}\label{fig:syn_on_off_less_rt}}
    \caption{ON/OFF request pattern.
    Client 1 sends 30 requests per minute (less than half of the capacity) during the ON phase and switches to OFF phase periodically.
    Client 2 is always in the ON phase, sending requests at a rate of 120 requests per minute (larger than half of the capacity).
    Requests have input lengths of 256 and output lengths of 256.}
    \label{fig:syn_on_off_less}
\end{figure}

\begin{figure}[ht]
    \subfloat[Received service rate (VTC).]{\includegraphics[width=0.23\textwidth]{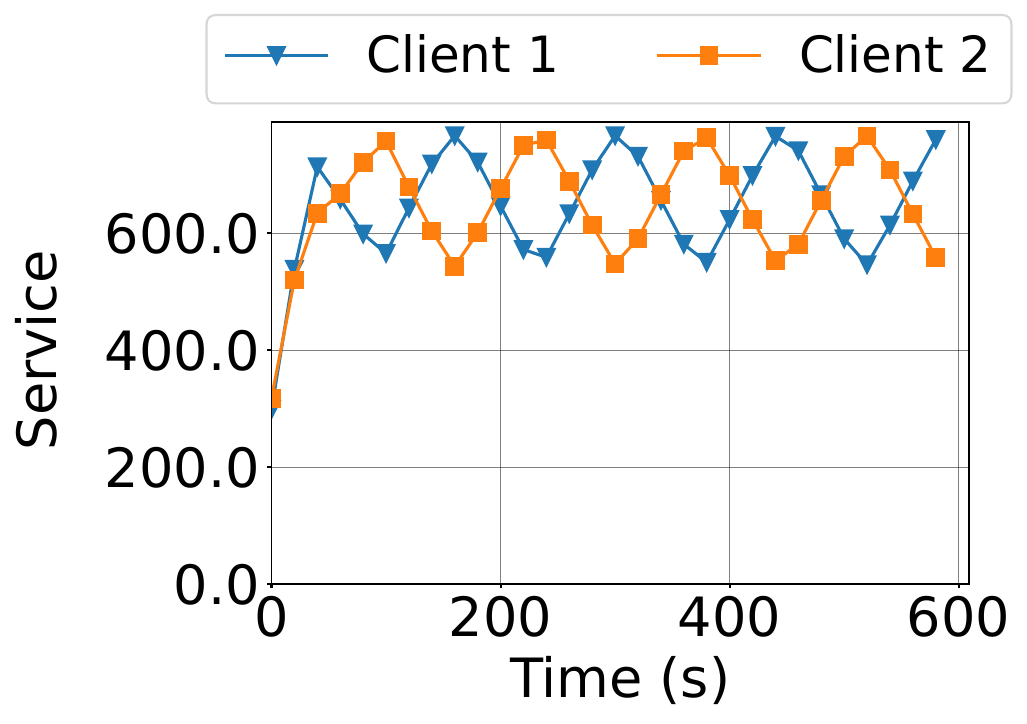}\label{fig:syn_on_off_ol_service}}
    \hfill
    \subfloat[Response time.]{\includegraphics[width=0.23\textwidth]{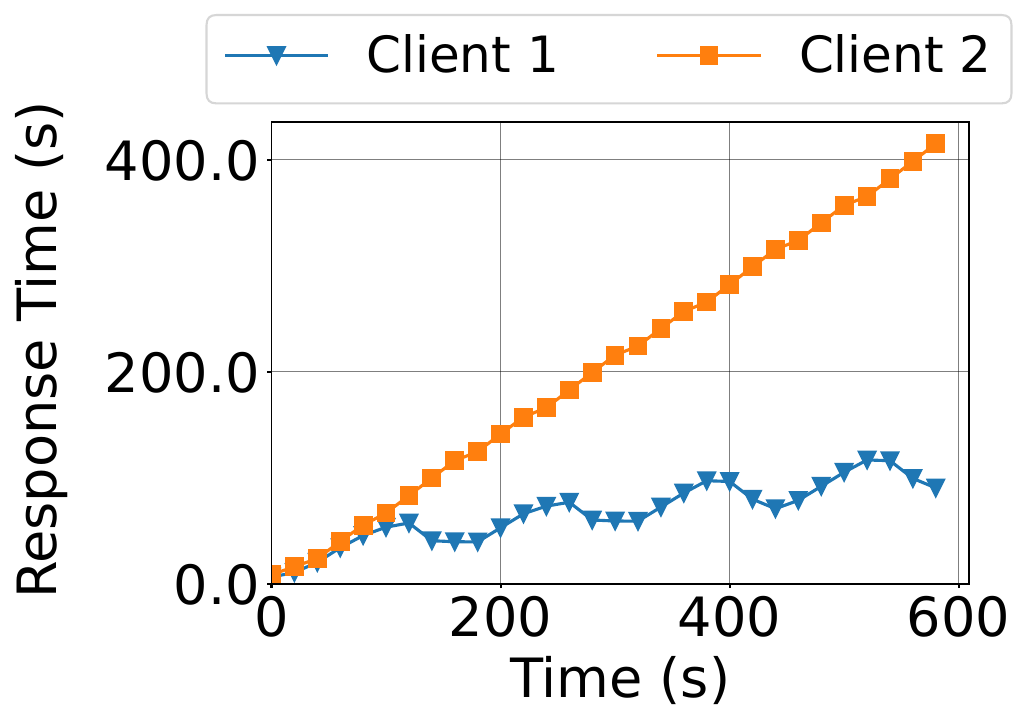}\label{fig:syn_on_off_ol_rt}}
    \caption{ON/OFF request pattern. Client 1 sends 120 requests per minute constantly during the ON phase (over its share), and stops sending during the OFF phase.
    Client 2 sends 180 requests per minute all the time (over its share).
    Requests have input lengths of 256 and output lengths of 256.}
    \label{fig:syn_on_off_overload}
    \vspace{-1em}
\end{figure}

\paragraph{Variable input/output length and poisson process}
In this experiment, we simulate scenarios where requests arrive stochastically. Furthermore, they send requests with different input and output lengths.
In both~\Cref{fig:syn_poisson_short_long} and~\Cref{fig:syn_poisson_short_long_2}, Client 1 sends requests with a high rate and Client 2 sends requests with a rate lower but still over its share.
Requests arrive according to a Poisson process with the coefficient of variance 1.
In \Cref{fig:syn_poisson_short_long}, client 1 sends short requests, and client 2 sends long requests.
In \Cref{fig:syn_poisson_short_long_2}, Client 1 sends requests with short input and long output, while Client 2 sends requests with long input and short output.
Similarly, with the observation before, VTC maintains a bounded difference between the services received by two clients. FCFS cannot preserve fairness according to \Cref{fig:syn_poisson_short_long_acc_service} and~\Cref{fig:syn_poisson_short_long_2_acc_service}.
This confirms that VTC can work under stochastic workloads with variable lengths.

\begin{figure}[t]
    \subfloat[Received service rate (VTC).]{\includegraphics[width=0.24\textwidth]{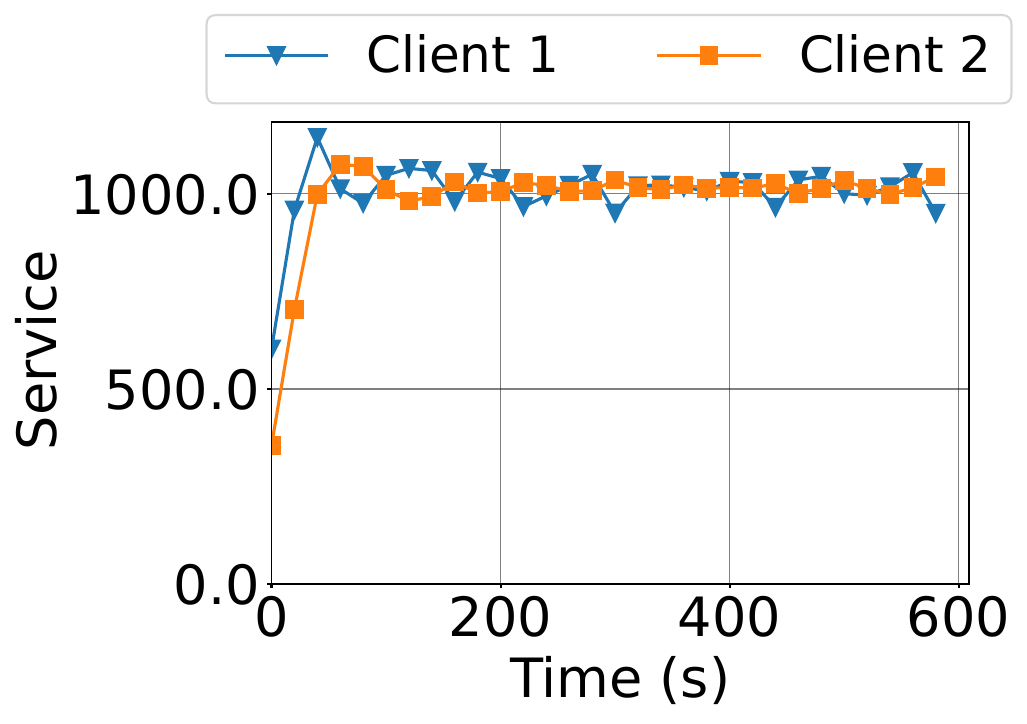}\label{fig:syn_poisson_short_long_service}}
    \hfill
    \subfloat[Absolute difference for accumulated service.]{\includegraphics[width=0.21\textwidth]{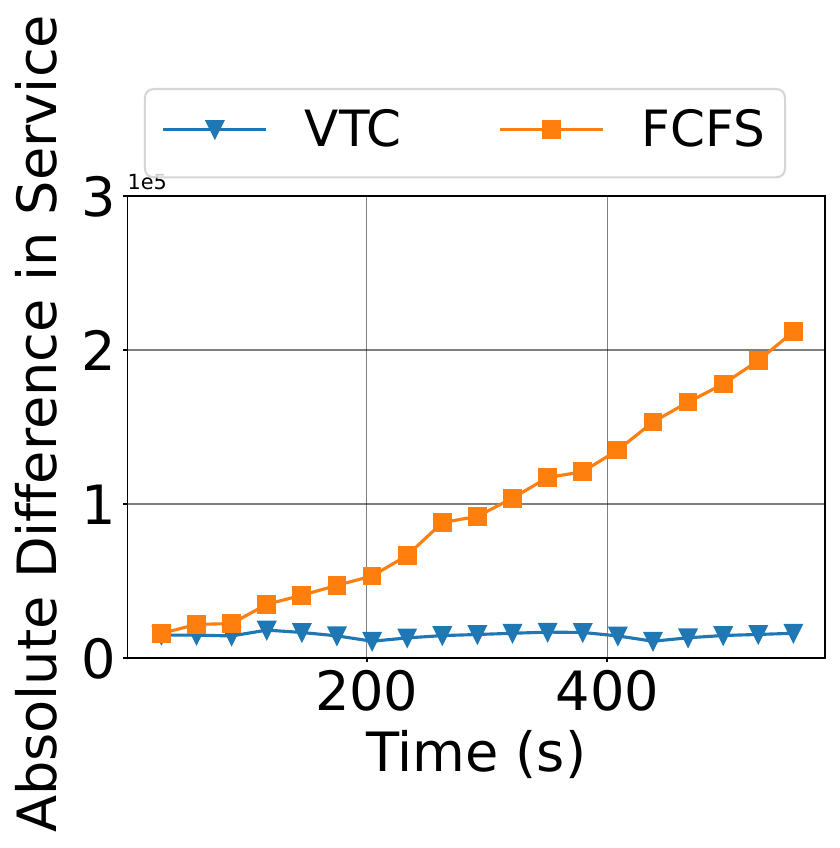}\label{fig:syn_poisson_short_long_acc_service}}
    \caption{Client 1 sends 480 requests per minute. Client 2 sends 90 requests per minute. Requests arrive according to a Poisson process with the coefficient of variance 1.
    Requests sent from Client 1 have input lengths of 64 and output lengths of 64.
    Requests sent from Client 2 have input lengths of 256 and output lengths of 256.}
    \label{fig:syn_poisson_short_long}
\end{figure}

\begin{figure}[t]
    \subfloat[Received service rate (VTC).]{\includegraphics[width=0.24\textwidth]{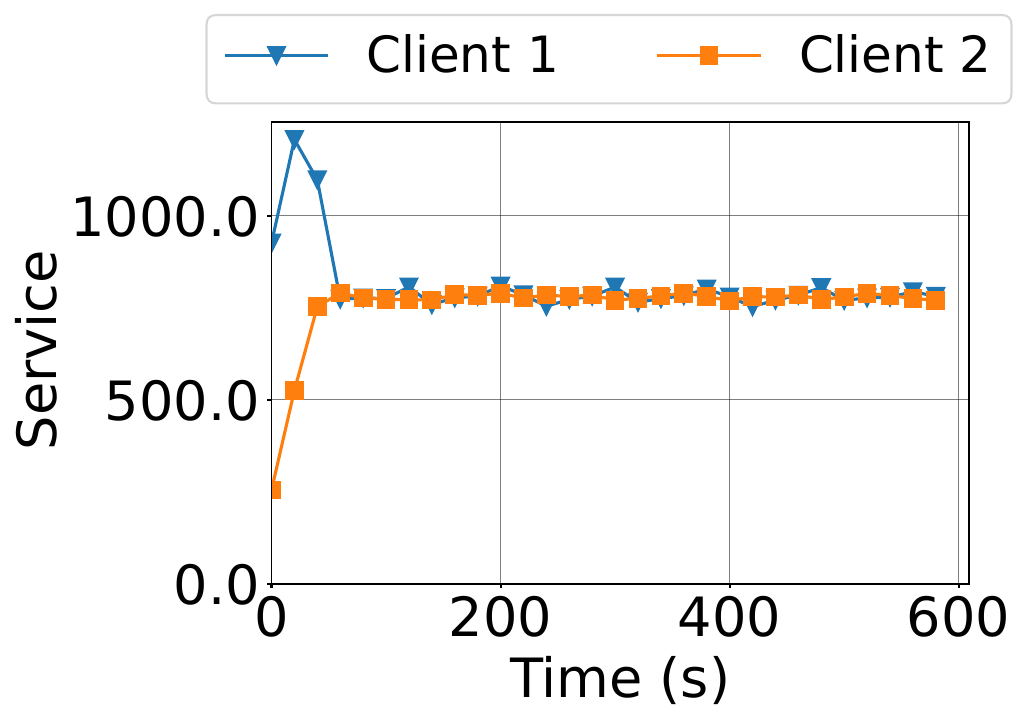}\label{fig:syn_poisson_short_long_2_service}}
    \hfill
    \subfloat[Absolute difference for accumulated service.]{\includegraphics[width=0.20\textwidth]{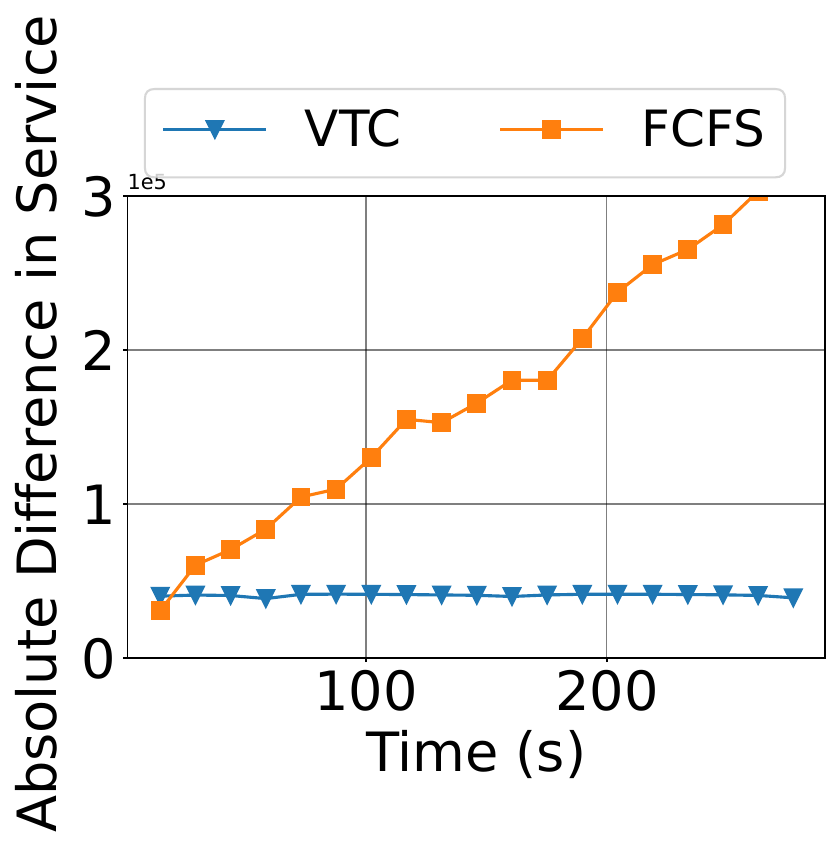}\label{fig:syn_poisson_short_long_2_acc_service}}
    \caption{Client 1 sends 480 requests per minute. Client 2 sends 90 requests per minute. Requests arrive according to a Poisson process with the coefficient of variance 1.
    Requests sent from Client 1 have input lengths of 64 and output lengths of 512.
    Requests sent from Client 2 have input lengths of 512 and output lengths of 64.}
    \label{fig:syn_poisson_short_long_2}
    \vspace{-1em}
\end{figure}

\paragraph{Isolation}
To illustrate the isolation property, we use the setup with a deterministic arrival pattern and the same input length and output length of 256.
In \Cref{fig:syn_increase}, Client 1 sends 30 requests per minute, which is under half of the server's capacity.
Client 2 acts as an "ill-behaved" client. It sends requests at a linearly increasing rate, and gradually over half of the system capacity. We observe that the response time of requests from client 1 is roughly unchanged, empirically validating the property stated in Theorem~\ref{thm:isolation}.

\begin{figure}[t]
    \subfloat[Received service rate (VTC).]{\includegraphics[width=0.23\textwidth]{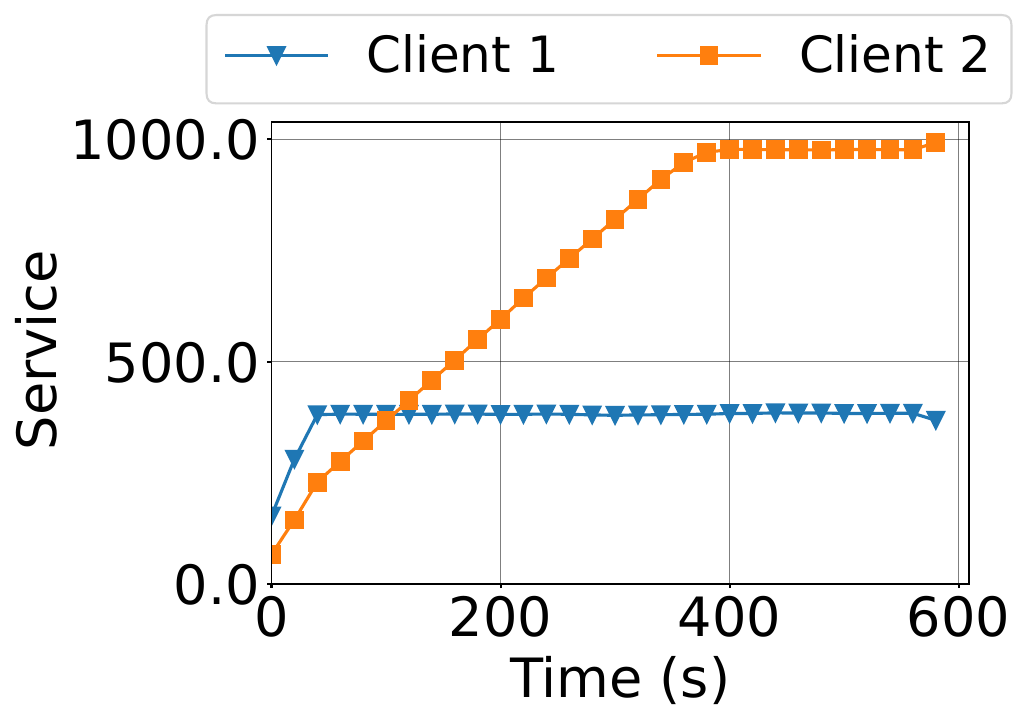}\label{fig:syn_increase_service}}
    \hfill
    \subfloat[Response time (VTC).]{\includegraphics[width=0.23\textwidth]{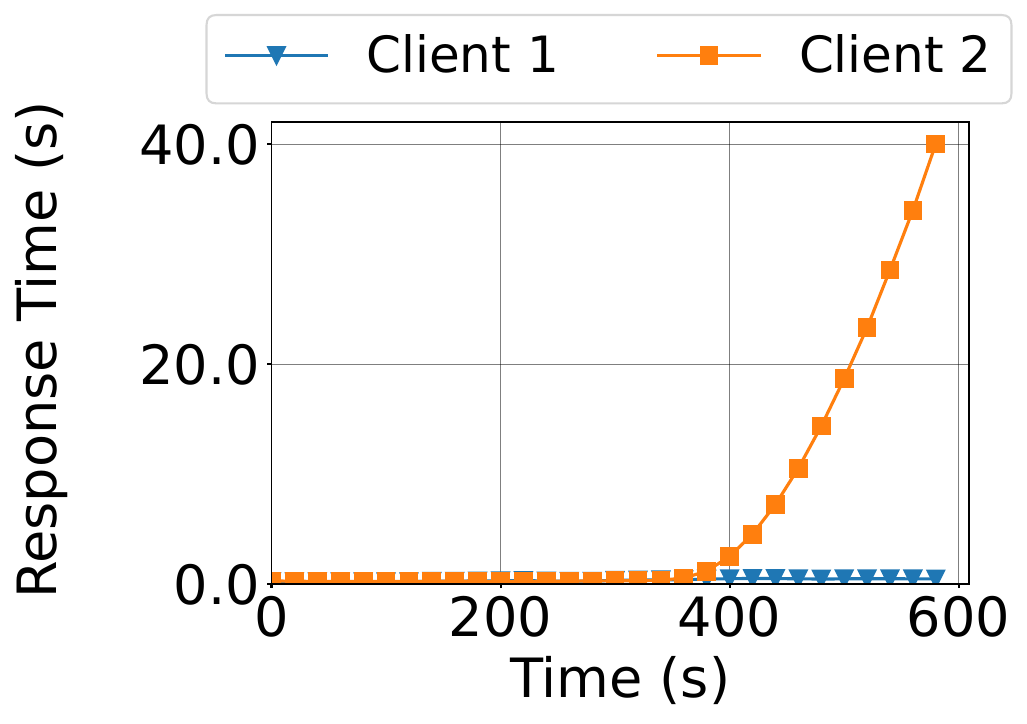}\label{fig:syn_increase_rt}}
    \caption{
    Client 1 sends 30 requests per minute, Client 2 sends 120 requests per minute, in a uniform arrival pattern.
    Requests have input lengths of 256 and output lengths of 256.
    Client 1 sends 30 requests per minute, which is under half of the server's capacity.
    Client 2 sends requests at a linearly increasing rate, and gradually over half of the system capacity.}
    \label{fig:syn_increase}
\end{figure}

\begin{figure}[t]
    \subfloat[Received service rate (VTC).]{\includegraphics[width=0.23\textwidth]{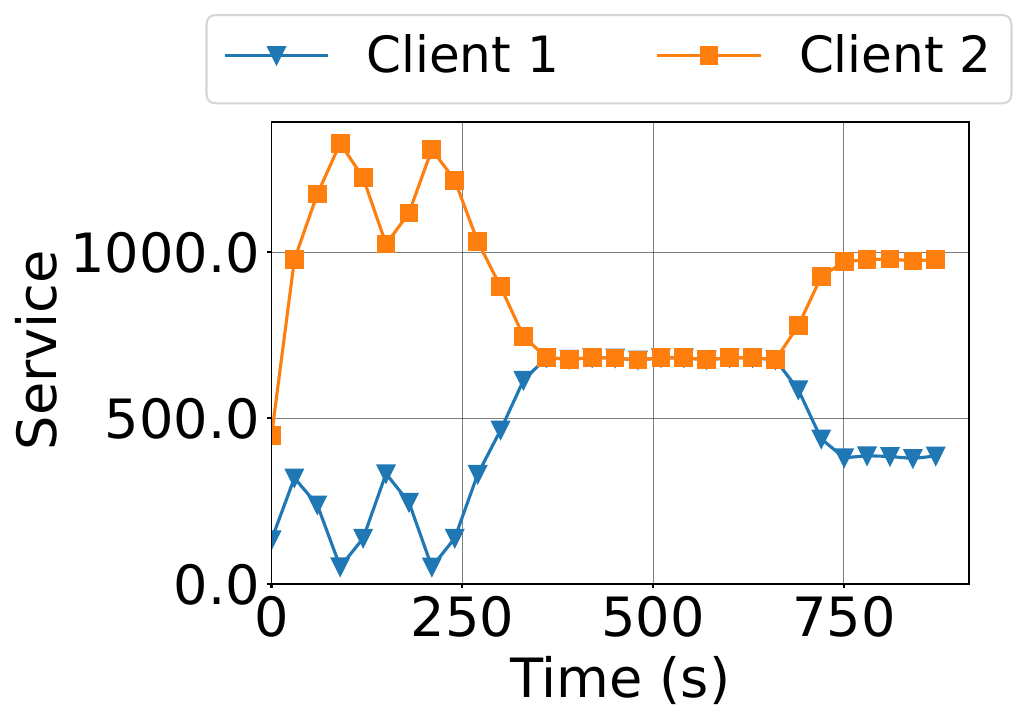}\label{fig:syn_shift_service}}
    \hfill
    \subfloat[Received service rate (LCF).]{\includegraphics[width=0.23\textwidth]{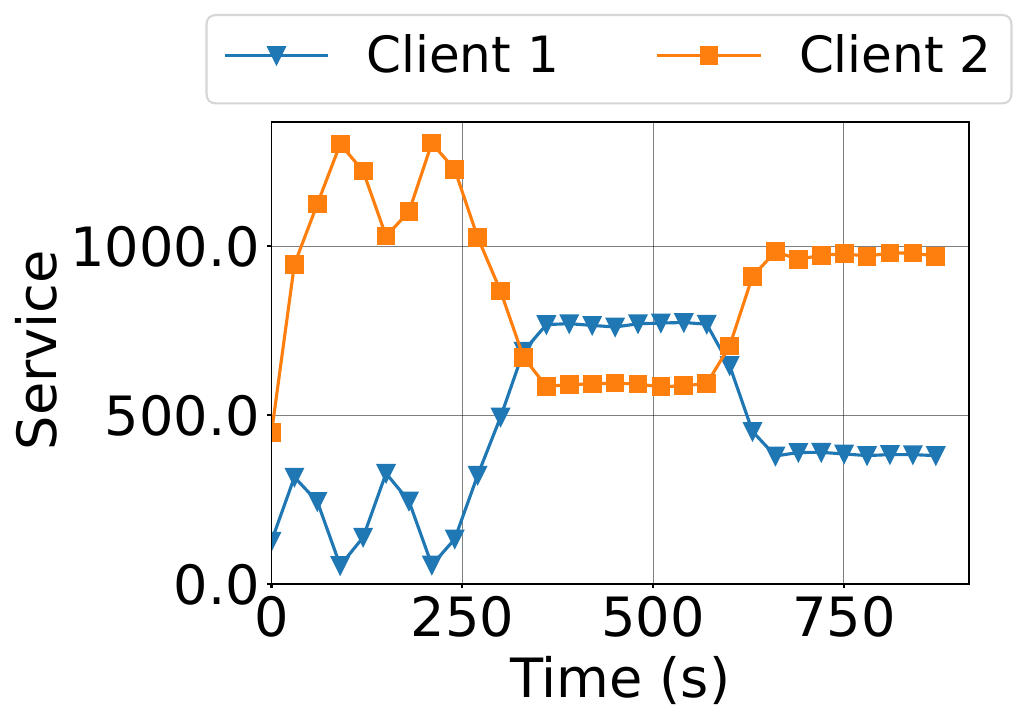}\label{fig:syn_shift_service_lcf}}
    \caption{Clients send requests in three phases, all with uniform arrival patterns.
    The first 5 minutes is ON/OFF phase.
    Client 1 sends 30 requests per minute during the ON phase (less than its share) and stops sending during the OFF phase. Each ON or OFF phase has 60 seconds.
    The second 5 minutes is the overload phase.
    Both Client 1 and Client 2 send 60 requests per minute, which causes the server to be overloaded.
    In the last 5 minutes, Client 1 sends 30 requests per minute (less than its share), and Client 2 sends 90 requests per minute, which causes the server to be still overloaded.
    Requests all have input lengths of 256 and output lengths of 256.}
    \label{fig:syn_shift}
    \vspace{-1em}
\end{figure}

\paragraph{Distribution shift}
In reality, clients' behavior may change over time.
To this end, we evaluate the robustness of VTC when the distribution of client requests shifts.
In \Cref{fig:syn_shift}, we construct a 15-minute workload comprising three phases.
The first phase is an ON/OFF phase, in which Client 1 sends requests less than its share only during the ON phase and stops during the OFF phase.
Client 2 sends requests at a constant rate, which makes the server overloaded.
We can observe the pattern for the first phase to be similar to \Cref{fig:syn_on_off_less_service}, which maintains a constant total service.
During the second phase, because the two clients both send requests over their share, a fair server should let them receive the same level of service.
\Cref{fig:syn_shift_service} demonstrates that VTC yields a desired pattern, similar to that shown in \Cref{fig:syn_overload_service}. \Cref{fig:syn_shift_service_lcf} reveals that LCF disproportionately serves Client 1, as it inherits Client 1's deficit from the first phase.
In the last phase, the serving pattern for VTC and LCF are similar, because they simply serve all requests from Client 1 immediately as Client 1 sends requests under its share.

\subsection{Results on Real Workloads}
We construct real workload traces from the traces of LMSYS Chatbot Arena~\cite{zheng2023judging, zheng2023lmsys}, following a similar process in~\cite{sheng2023slora}.
The trace is from a server that serves multiple LLMs. To adapt it to our setting, we treat each LLM as a client. In total, there are 27 clients.
To sample from this log, we define $D$, the duration, and $R$, the request rate. We then sample $R*D$ requests from the trace, and re-scale the real-time stamps to $[0, D]$. We use a duration of 10 minutes to be consistent with previous experiments, and a request rate of 210 requests per minute for the whole system.
With the adapted workload, we run Llama-2-7b on A10G (24GB).
In summary, the prompts from the 27 clients are collected from the real world interactions, which will be sent to the server for inference on Llama-2-7b.
The timestamps are re-scaled from the real-world trace.

For better visualization of the evaluation results, we select two clients that send the most requests and two clients that send a medium number of requests. We sort 27 clients according to the number of requests they send, and depict the statistics of the $13^{th}, 14^{th}$ and $26^{th}, 27^{th}$ clients. We do not choose clients that send the least requests because they typically only send requests in a small interval.

\begin{figure}[t]
    \centering
    \begin{subfigure}[b]{0.23\textwidth}
    \includegraphics[width=\textwidth]{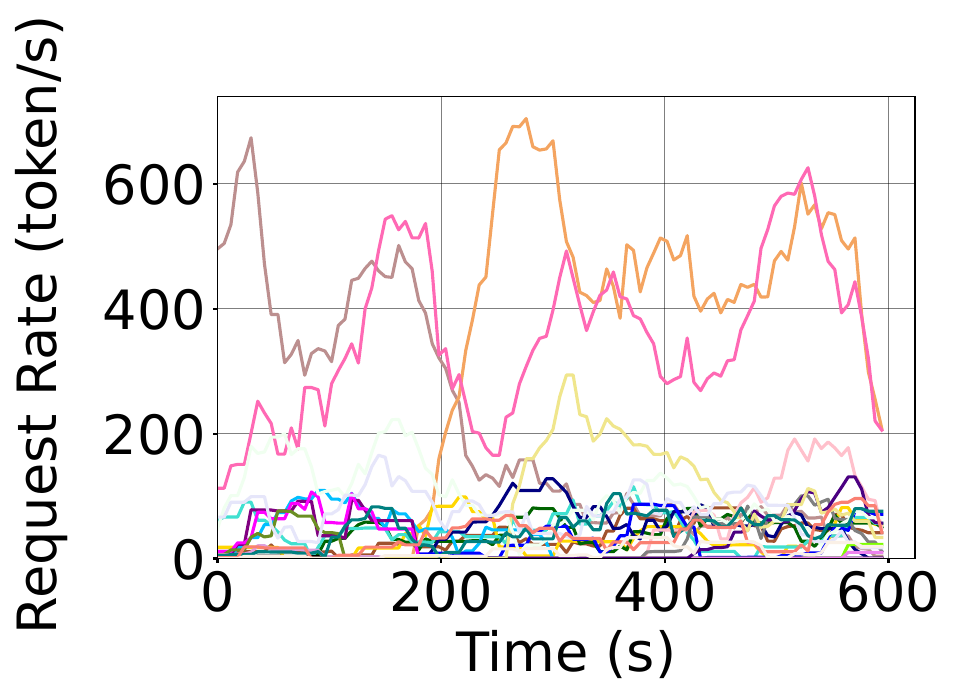}
    \end{subfigure}
    \hfill
    \begin{subfigure}[b]{0.23\textwidth}
    \includegraphics[width=\textwidth]{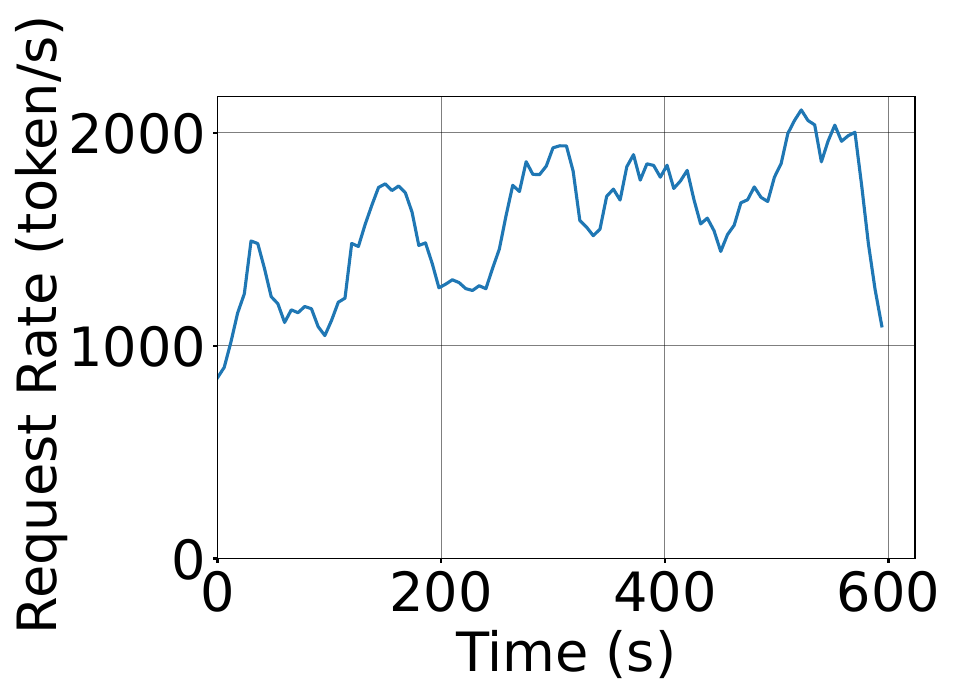}
    \end{subfigure}
    \caption{Request rate distribution during the sampled 10 minutes duration with re-scale. The figure on the left denotes the real-time request rate for the 27 clients. A few clients have sent many more requests than others, reflecting the original trace of a few most popular models. The figure on the right depicts the total request rate from all 27 clients.}
    \label{fig:real_req_rate}
\end{figure}

\begin{figure}[t]
    \centering
    \begin{subfigure}[b]{0.23\textwidth}
    \includegraphics[width=\textwidth]{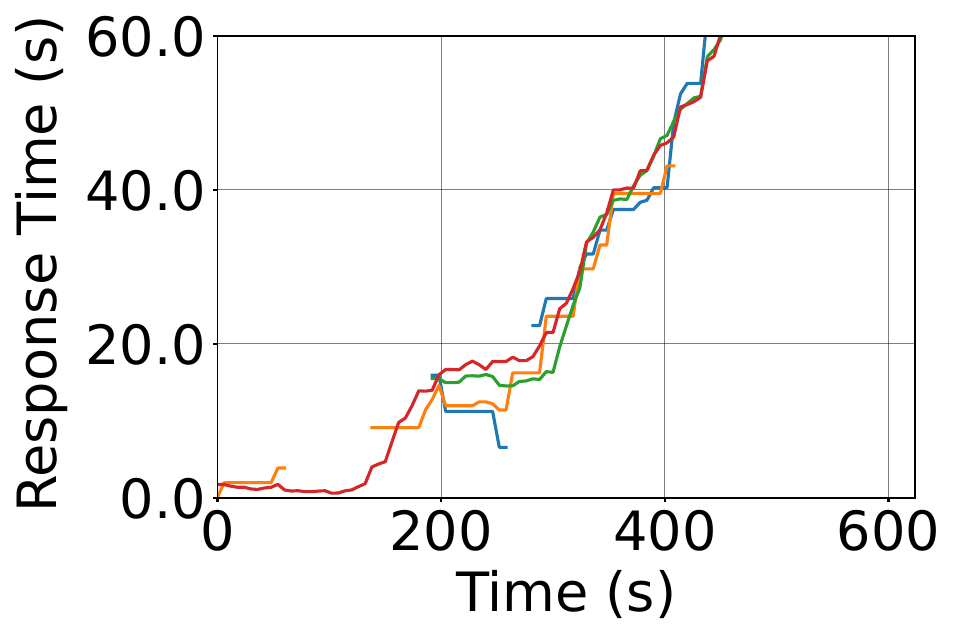}
    \end{subfigure}
    \hfill
    \begin{subfigure}[b]{0.23\textwidth}
    \includegraphics[width=\textwidth]{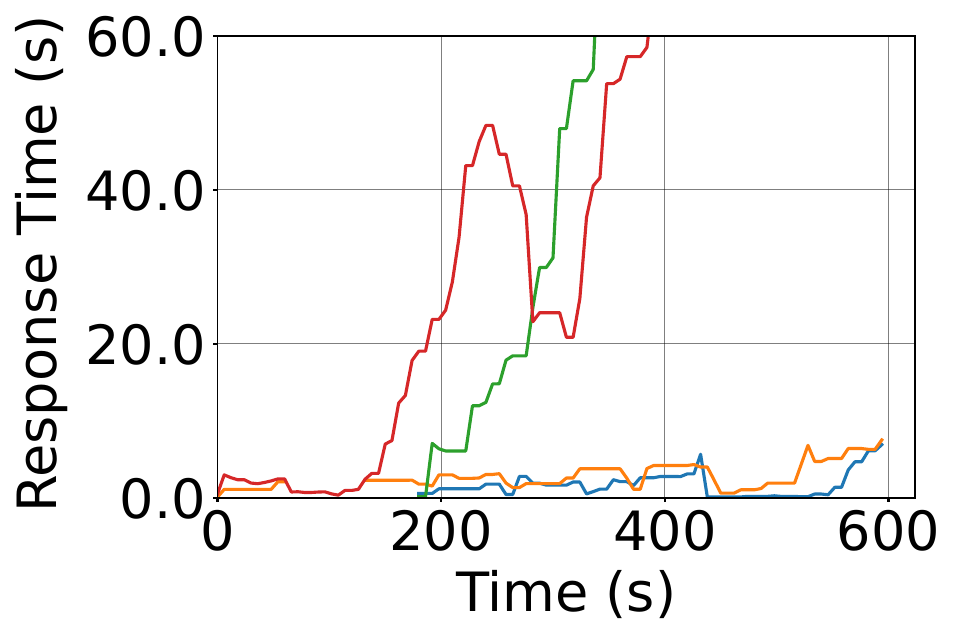}
    \end{subfigure}
    \caption{Response time of 4 selected clients when using FCFS (Left) and VTC (Right) in real traces. Each curve corresponds to one client. There are some curves that show disconnected because, during some periods, a client may have no requests served. Requests distribution see \Cref{fig:real_req_rate}.}
    \label{fig:real_fcfs_throughput_response_time}
\end{figure}

\begin{figure}[t]
    \centering
    \begin{subfigure}[b]{0.23\textwidth}
    \includegraphics[width=\textwidth]{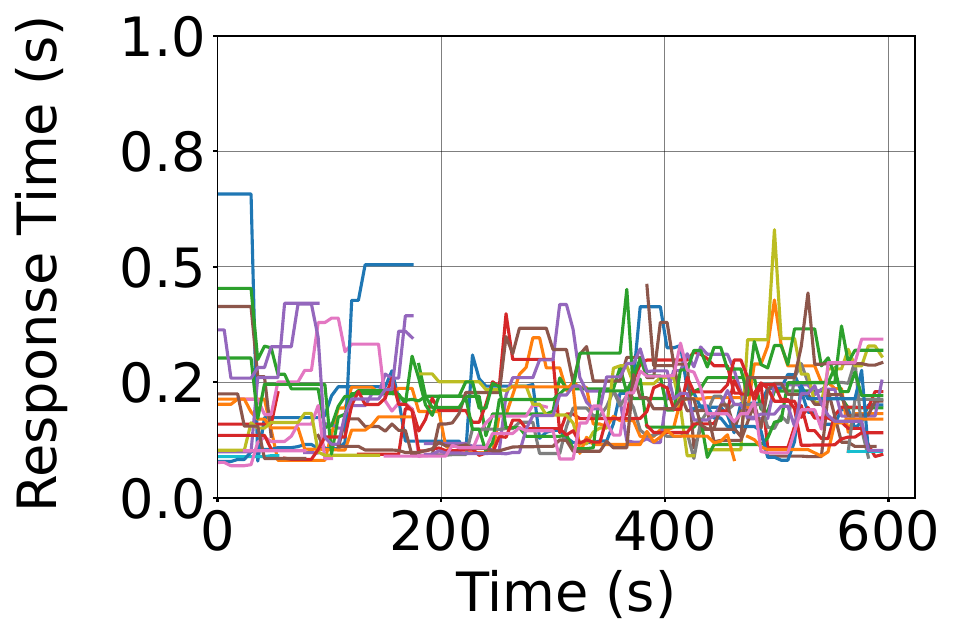}
    \end{subfigure}
    \hfill
    \begin{subfigure}[b]{0.23\textwidth}
    \includegraphics[width=\textwidth]{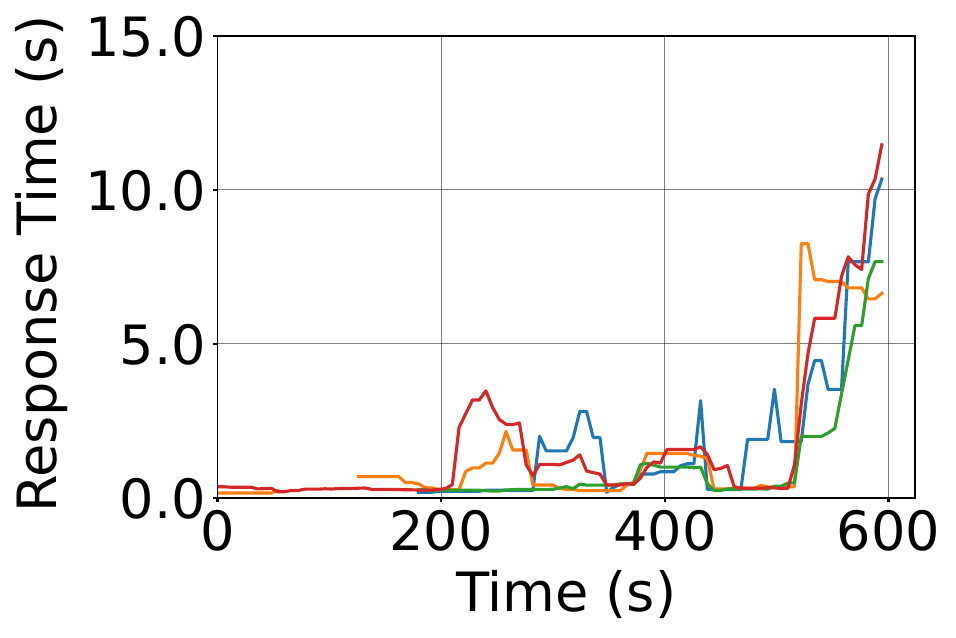}
    \end{subfigure}
    \hfill
    \begin{subfigure}[b]{0.23\textwidth}
    \includegraphics[width=\textwidth]{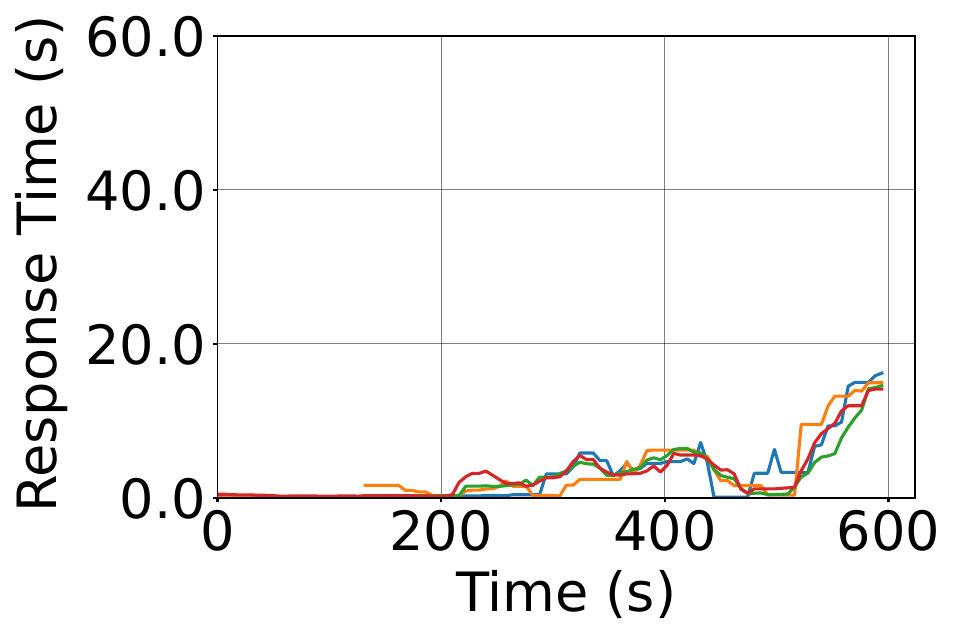}
    \end{subfigure}
     \begin{subfigure}[b]{0.23\textwidth}
    \includegraphics[width=\textwidth]{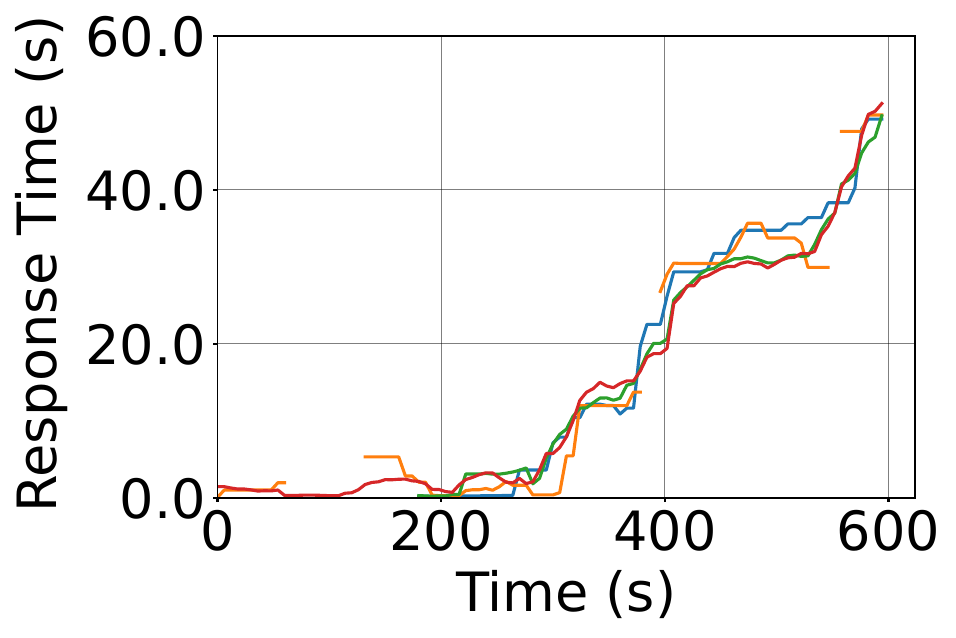}
    \end{subfigure}
    \hfill
    \caption{Response time of 4 selected clients (all 27 clients when rpm=5) when using RPM in real traces. Left-upper to right-bottom corresponds to a different rate limit (5, 15, 20, 30 requests per minutes, respectively). There are some curves that show disconnected because, during some periods, a client may have no requests served.}
    \label{fig:real_rpm_throughput_response_time}
\end{figure}

\begin{figure}[t]
    \centering
    \begin{subfigure}[b]{0.25\textwidth}
    \includegraphics[width=\textwidth]{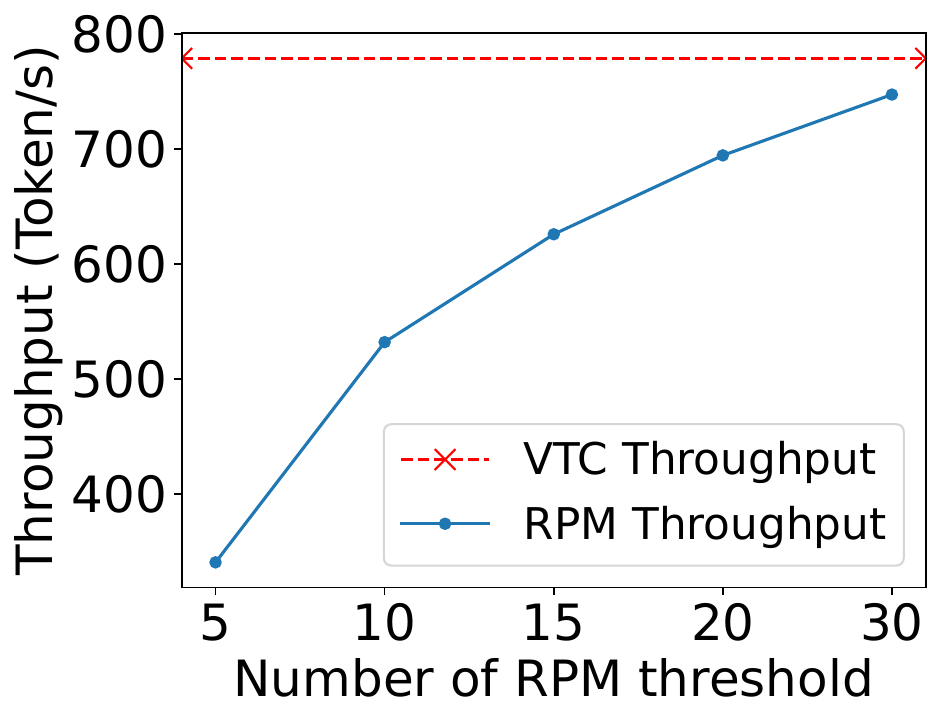}
    \end{subfigure}
    \caption{Throughput of RPM versus different number of requests per minute threshold. Compared with VTC, RPM consistently exhibits a lower throughput. }\label{fig:real_rpm_overall_throughput_comparison}
\end{figure}

\textbf{Request distribution}
The request rate distribution is visualized in \Cref{fig:real_req_rate}.
The request rate of individual clients and the total request rate are all highly dynamic.
The input and output length distribution is depicted in \Cref{fig:real_req_len} in the Appendix.
The average input length is 136, and the average output length is 256. The input and output lengths have the range of $[2, 1021]$ and $[2, 977]$, respectively.

\textbf{Effect on response time} 
Figure~\ref{fig:real_fcfs_throughput_response_time} shows the response time of 4 selected  clients on the real trace. With FCFS scheduling, the response time of all clients increases drastically because some clients send over their share, monopolizing the service and impacting other clients. With VTC, only clients that send requests over its share will have a drastic increase in the response time.

\textbf{Analysis of request rate per limit approach}
In~\Cref{fig:real_rpm_throughput_response_time}, we show the response of RPM approach with different rate limits. In~\Cref{fig:real_rpm_overall_throughput_comparison}, we show the corresponding throughput comparison with VTC. These plots reveal a core dilemma of the RPM approach - the system has to choose between fairness or throughput, but not both. If the rate limit is low, then the system rejects many requests from clients that send over their share.
This opens the capacity for clients with fewer requests. As demonstrated in the uppermost plot in \Cref{fig:real_rpm_throughput_response_time}, all requests have a similar response time.
However, this low rate limit rejects more requests than needed, causing a lower throughput (cluster-wise throughput is $\approx 340$ output tokens per second when RPM=5, as opposed to $\approx 779$ tokens per second in VTC or FCFS). When the rate limit is set higher, the system throughput is gradually increasing, i.e., increasing from $340$ tokens to $747$ tokens per second. However, the response time for all requests grows up. When the request rate is set higher and higher, the response time curve converges to the one in FCFS, and there is no fairness guarantee anymore.
In other words, the RPM approach can be summarized as follows: it functions as an FCFS (First-Come, First-Served) approach with admission control (rate limiting), rather than as a truly fair scheduler. Its fairness is achieved by rejecting numerous requests from other clients, which compromises the overall system throughput.

\textbf{Quantitative Measurement}
We measured the maximum and average service difference described in \Cref{sec:setup} during the time window (10 minutes) in which we ran the experiments.
\Cref{tab:quant_fairness} is a summary for all baselines using this quantitative measurement for real workload trace.

\begin{table}[ht]
\centering
\resizebox{1\columnwidth}{!}{
\begin{tabular}{c|ccccc}
\toprule
Scheduler & Max Diff & Avg Diff & Diff Var & Throu & Isolation\\
\midrule
FCFS & 759.97 & 433.53 & 32112.00 & 777 & No\\
LCF & 750.49 & 323.82 & 29088.90 & 778 & Some\tablefootnote{LCF achieves isolation if the workload does not change. However, the isolation can be broken by newly joined clients whose virtual counter is lagging behind.}\\
\midrule
VTC & 368.40 & 251.66 & 6549.16 & 779 & Yes\\ 
VTC(predict) & 365.47 & 240.33 & 5321.62 & 773 & Yes\\
\textbf{VTC(oracle)} & \textbf{329.46} & \textbf{227.51} & \textbf{4475.76} & \textbf{781} & \textbf{Yes}\\ 
\midrule
RPM(5) & 143.86 & 83.58 & 1020.46 & 340 & Some\\
RPM(20) & 446.76 & 195.71 & 7449.79 & 694 & Some\\
RPM(30) & 693.66 & 309.45 & 24221.31 & 747 & Some\\
\bottomrule
\end{tabular}
}
\caption{The service difference is counted by summing the service difference between each client and the client who received the maximum services.
Throughput is the total number of tokens (including input and output tokens) processed divided by the total execution time.}
\vspace{-1em}
\label{tab:quant_fairness}
\end{table}

\begin{figure}[ht!]
    \subfloat[Different memory pool size.]{\includegraphics[width=0.23\textwidth]{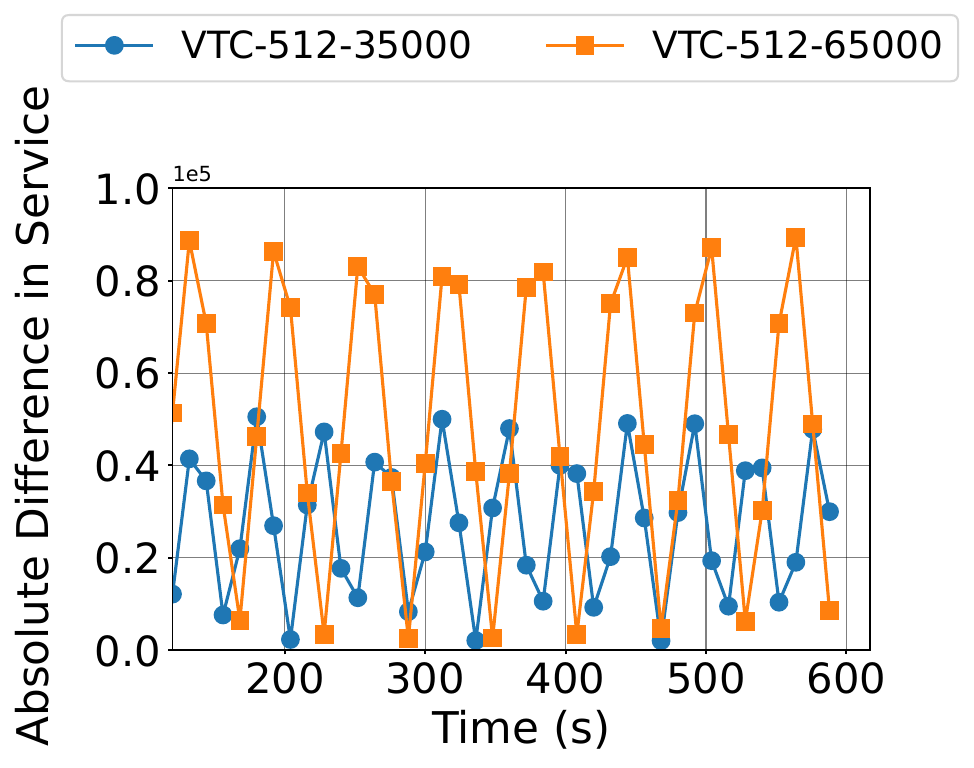}\label{fig:syn_overload_s4_acc_service_mem_pool}}
    \hfill
    \subfloat[Different request length.]{\includegraphics[width=0.23\textwidth]{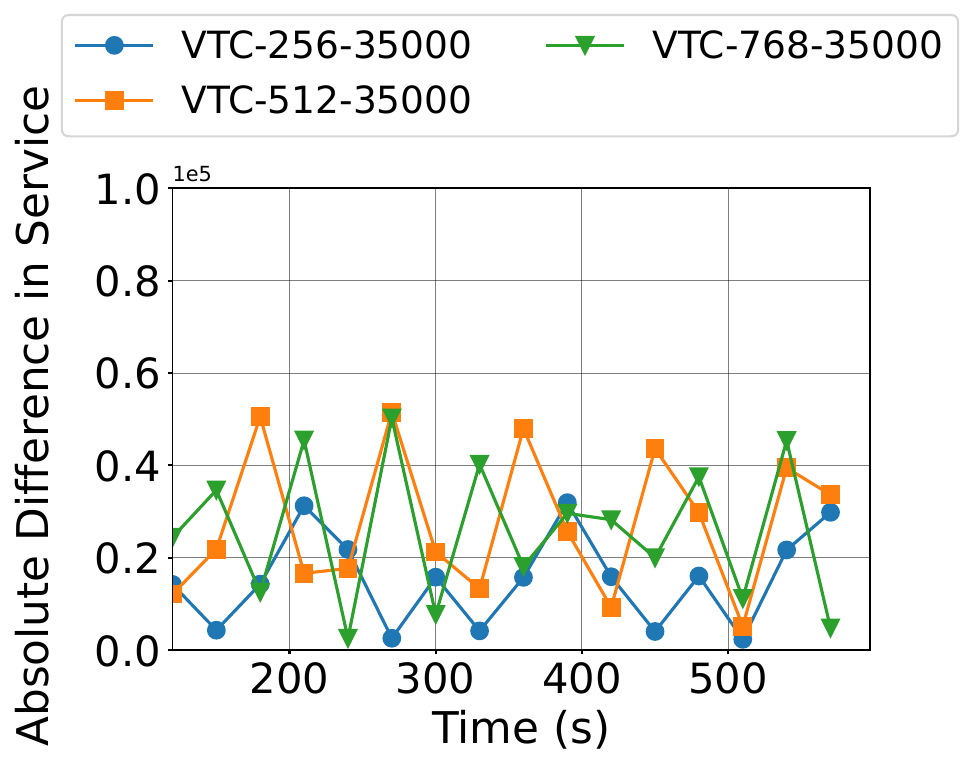}\label{fig:syn_overload_s4_acc_service_req_len}}
    \caption{In all settings, both clients are sending requests of the same lengths with uniform arrival patterns. They send requests with different request rates but are both backlogged. Three different request lengths (256*2, 512*2 and 768*2) are evaluated for the 35000 KV cache setting.}
    \label{fig:syn_overload_s4}
    \vspace{-0.5em}
\end{figure}

\subsection{Ablation Study}
\label{sec:ablation}
In~\Cref{fig:syn_overload_s4}, we evaluate how different memory pool sizes and request lengths will affect scheduling fairness. As shown in~\Cref{fig:syn_overload_s4_acc_service_mem_pool}, with a larger memory pool size, the attainable batch size becomes larger. Therefore, there is greater variation in the absolute difference of accumulated services received by the clients when the memory pool is 65000 than that is 35000, which empirically validates~\Cref{thm:main_fairness}.
\Cref{fig:syn_overload_s4_acc_service_req_len} demonstrates that larger request lengths will also lead to greater variations in the service difference.
This is caused by the unknown output length of request generation. At line 24 in \Cref{alg:vtc}, the most conservative way of only counting the input tokens leads to over-compensation for the smallest counter, as all the potential output tokens are not counted.
A shorter request length has a milder effect of over-compensation.
The curves of $(512*2)$ and $(768*2)$ show the same variance.
This is because at length $(512*2)$, the upper bound given by VTC has been reached.

\section{Related Works}
\paragraph{Fairness in scheduling}
Achieving fairness in scheduling resources in a multi-client environment has been a long-standing topic in computer science~\cite{schwiegelshohn2000fairness, radunovic2007unified, zaharia2010delay, drf}. Among these, Fair Queuing~\cite{fq} has been adapted into many variants for different contexts such as CPU scheduling~\cite{cfq}, link bandwidth allocation~\cite{wfq,pgps,sfq,sfq_d,self-clock}, and memory allocation~\cite{nesbit2006fairmemory}. Deficit round robin~\cite{drr} and stochastic fair queuing~\cite{McKenney1990stochasticfq} are non-real-time fair queuing algorithms for variable-size packets, providing guarantees for long-term fairness.
There are also real-time fair queuing algorithms (e.g., WFQ~\cite{wfq} and SFQ~\cite{sfq}) that can make more strict short-term delay guarantees~\cite{dorlan2016introduction}.
Our scheduling algorithm is different from these algorithms because we need to consider the batching effects across multiple clients' requests and deal with unknown request length. Further, we need to accommodate a flexible notion of fairness on both performance and GPU resource consumption.

\paragraph{Fairness in ML training}
Within the realm of deep learning, research has delved into scheduling jobs in shared clusters~\cite{mahajan2020themis,narayanan2020heterogeneity,chaudhary2020balancing,qiao2021pollux}, with a primary focus on long-duration training jobs.
Machine Learning training jobs have unique characteristics and traditional fair schedulers~\cite{isard2009quincy,robert2016carbyne} designed for big-data workflow usually fail~\cite{mahajan2020themis}. In particular, Themsis~\cite{mahajan2020themis} points out that ML jobs are device placement sensitive, where jobs will be envious of other's placement even if they are assigned the same number of resources. It then defines a finish-time fairness metric to measure fairness in ML training scenarios. Pollux~\cite{qiao2021pollux} further points out that ML jobs should jointly consider the throughput and the statistical efficiency, and develop a goodput-based scheduler that further improves the finish-time fairness of ML jobs. In this paper, we consider fairness in LLM serving. The fairness problem in LLM serving is quite different from the fairness problem in model training. In model training, different clients' GPUs are isolated and the problem is which GPUs are assigned to each client. Achieving fairness in LLM serving requires design for a different set of issues, including how to batch requests from multiple clients to achieve high GPU utilization.

\paragraph{LLM Serving Systems}
How to improve the performance of LLM serving systems has recently gained significant attention.
Notable techniques cover advanced batching mechanisms~\cite{fang2021turbotransformers,yu2022orca}, memory optimizations~\cite{sheng2023flexgen,kwon2023efficient},
GPU kernel optimizations~\cite{wang2021lightseq,aminabadi2022deepspeed,nvidiaft,dao2023flashattention},
model parallelism~\cite{pope2023efficiently,aminabadi2022deepspeed,li2023alpaserve}, parameter sharing~\cite{zhou2022pets}, and speculative execution~\cite{stern2018blockwise,miao2023specinfer} were proposed.
FastServe~\cite{wu2023fast} explored preemptive scheduling to minimize job completion time (JCT).
However, none of these works consider fairness among clients.
Our work bridges this gap, and our proposed scheduling methods can be easily integrated with many of these techniques. Our implementation used for this paper is built atop continuous batching (iteration-level scheduling)~\cite{yu2022orca}\footnote{We presented VTC on continuous batching with separated prefill and decode steps in the main context. A general integration is discussed in \Cref{sec:vtc_integration}.} and PagedAttention~\cite{kwon2023efficient}.

\section{Conclusion}
We studied the problem of fair serving in Large Language Models (LLMs) with regard to the service received by each client.
We identified unique characteristics and challenges associated with fairness in LLM serving, as compared to traditional fairness problems in networking and operating systems.
We then defined what constitutes fairness and proposed a fair scheduler, applying the concept of fair sharing to the domain of LLM serving at the token granularity.

\section*{Acknowledgment}
We thank Aurick Qiao, Shu Liu, Stephanie Wang, Hao Zhang for helpful discussions and feedback.
This research was supported by gifts from Anyscale, Astronomer, Google, IBM, Intel, Lacework, Microsoft, Mohamed Bin Zayed University of Artificial Intelligence, Samsung SDS, Uber, and VMware.
Ying is partly supported by the Stanford Center for Automated Reasoning.
We thank Clark Barrett for academic advising and funding support.

\bibliographystyle{plain}
\bibliography{reference}

\newpage
\newpage
\appendix
\section{Missing Proofs in Proving Fairness of VTC}
\label{sec:proof}

\begin{lemma}
\label{lem:min_non_decrease}
In \Cref{alg:vtc}, $\min_{i\in Q}(c_i)$ is non-decreasing during the time when $Q\neq \emptyset$.
\end{lemma}
\begin{proof}
We prove the lemma by case study on each line of changing the $c_i$'s.
\begin{itemize}
    \item In the initialization, all $c_i=0$, lemma holds.
    \item If the condition of line 7 is satisfied, at lines 8-14, a new client will be added to $Q$.
    If lines 9-10 are reached, the $\min_{i\in Q}(c_i)$ is equals to its value at the last time when $Q\neq \emptyset$.
    If lines 12-13 are reached, since $c_u = \max\{c_u, \min_{i\in Q} c_i\}$, the $\min_{i\in Q}(c_i)$ will not change.
    \item At line 24 and line 30, the $c_i$'s can only increase, so that $\min_{i\in Q}(c_i)$ is non-decreasing.
    \item At line 26, if a client has cleared all its requests from $Q$, that the client is removed from $Q$, $\min_{i\in Q}(c_i)$ cannot decrease.
\end{itemize}
\end{proof}

\invariant*
\begin{proof}
We prove the lemma by induction.
During the induction, for each line of change of $c_i$ in \Cref{alg:vtc}, we use $c_i'$ to denote the new value and $c_i$ to denote the original value.
Similarly, we use $Q'$ to donate the new value and $Q$ to denote the original value.
We also use $c_i^{(t)}$ to denote the value of $c_i$ at time $t$, and $Q^{(t)}$ to denote the value of $Q$ at time $t$.
\begin{enumerate}
    \item In the initialization, all $c_i=0$, Equation (\ref{eq:invariant}) holds.

    \item If a client $u\notin Q$ receive a new request and thus $Q'=Q\cup \{u\}$, line 12-13 will be reached, and thus $c'_u = \max\{c_u, \min_{i\in Q} c_i\} \geq \min_{i\in Q} c_i$.
    Then we have,
    \begin{equation}
    \min_{i\in Q'}c'_i = \min\{c_u', \min_{i\in Q} c_i\} = \min_{i\in Q} c_i.
    \label{eq:start_backlog_min_bound}
    \end{equation}
    Let $t$ be the last time that $u$ was in $Q$ before the change, and thus $c_u^{(t)} = c_u$.
    From \Cref{eq:invariant}, there is
    \[ \max_{i\in Q^{(t)}} c_i^{(t)} - \min_{i\in Q^{(t)}} c_i^{(t)} \leq \max (w_p \cdot L_{input}, w_q\cdot M). \]
    Then we have
    \[ c_u^{(t)} \leq \max_{i\in Q^{(t)}} c_i^{(t)} \leq \min_{i\in Q^{(t)}} c_i^{(t)} + \max (w_p \cdot L_{input}, w_q\cdot M). \]
    From \Cref{lem:min_non_decrease}, there is
    $ \min_{i\in Q^{(t)}} c_i^{(t)} \leq \min_{i\in Q} c_i $, so we have
    \[ c_u = c_u^{(t)} \leq \min_{i\in Q} c_i + \max (w_p \cdot L_{input}, w_q\cdot M), \]
    which can derive
    \[ c_u' = \max\{c_u, \min_{i\in Q} c_i\} \leq \min_{i\in Q} c_i + \max (w_p \cdot L_{input}, w_q\cdot M). \]
    Combine with \Cref{eq:invariant} and \Cref{eq:start_backlog_min_bound}, there is
    \begin{align*}
        \max_{i\in Q'} c_i' &= \max\{ c_u', \max_{i\in Q} c_i \} \\ &\leq \min_{i\in Q} c_i + \max (w_p \cdot L_{input}, w_q\cdot M)\\
        &\leq \min_{i\in Q'} c'_i + \max (w_p \cdot L_{input}, w_q\cdot M)
    \end{align*}
    Therefore, \Cref{eq:invariant} holds after the change.

    \item If a client $u$ is left from $Q$ at line 26, the difference $\max_{i\in Q} c_i - \min_{i \in Q} c_i$ will not increase. Because $\max(C') - \min(C') \leq \max(C) - \min(C), \forall C \supseteq C', C'\neq \emptyset$.
    Therefore, Equation (\ref{eq:invariant}) still holds.

    \item At line 24, since $c_k = \min_{i\in Q} c_i$, there is 
    \begin{equation}
        \label{eq:line18_min_bound}
        \min_{i\in Q} c_i \leq \min_{i\in Q} c'_i \leq c_k' \leq \min_{i\in Q} c_i + w_p\cdot L_{input}.
    \end{equation}
    From Equation (\ref{eq:invariant}), we have 
        \[\max_{i\in Q} c_i \leq \min_{i \in Q} c_i + \max (w_p \cdot L_{input}, w_q\cdot M).\]
    Because:
    \[\max_{i\in Q} c_i' = \max(\max_{i\in Q} c_i, c_k')\]
    We have:
    \begin{equation}
        \label{eq:line18_max_bound}
        \max_{i\in Q} c_i' \leq \max(\min_{i \in Q} c_i + \max (w_p \cdot L_{input}, w_q\cdot M), c_k')
    \end{equation}
    In \Cref{eq:line18_min_bound} we have derived that: 
    \[c_k' \leq \min_{i\in Q} c_i + w_p\cdot L_{input}\]
    Thus:
    \begin{align*}
    c_k' &\leq \min_{i\in Q} c_i + w_p\cdot L_{input}\\
    &\leq \min_{i\in Q} c_i + \max (w_p \cdot L_{input}, w_q\cdot M)
    \end{align*}
    Thus:
    \[\max(\min_{i\in Q} c_i + \max (w_p \cdot L_{input}, w_q\cdot M), c_k') = \] 
    \[\min_{i \in Q} c_i + \max (w_p \cdot L_{input}, w_q\cdot M).\]
    Thus  \Cref{eq:line18_max_bound} gives:
    \begin{equation}
        \max_{i\in Q} c_i' \leq \min_{i \in Q} c_i + \max (w_p \cdot L_{input}, w_q\cdot M)
    \end{equation}
    Finally, combining the inequality from \Cref{eq:line18_min_bound} that \[\min_{i \in Q} c_i \leq \min_{i \in Q} c_i',\] we arrive at:
    \[\max_{i\in Q} c_i' \leq \min_{i\in Q} c'_i + \max (w_p \cdot L_{input}, w_q\cdot M). \]    
    Therefore, \Cref{eq:invariant} holds.

    \item At line 30,
    let $k = \argmax_{i \in Q} c_i'$, so that $c_k' = \max_{i\in Q} c_i'$.
    Let $r$ be the last one among requests from $k$ that have been scheduled.
    Let $t$ be the time when $r$ was selected at line 21.
    Since $r$ is the last one been scheduled from $k$, there is
    \begin{equation}
    \max_{i\in Q} c_i' = c'_k \leq c_k^{(t)} + w_q\cdot M
    \label{eq:decode_fair_ub}
    \end{equation}

    Because request $r$ from client $k$ has been scheduled at time $t$, from line 20, there is $c_k^{(t)} = \min_{i\in Q^{(t)}} c_i^{(t)}$.
    From Lemma~\ref{lem:min_non_decrease}, we have $\min_{i\in Q^{(t)}} c_i^{(t)} \leq \min_{i\in Q} c'_i$.
    Combine with \Cref{eq:decode_fair_ub}, we have
    \[ \max_{i\in Q} c'_i - \min_{i\in Q} c'_i \leq w_q\cdot M. \]
    Therefore, \Cref{eq:invariant} holds.
\end{enumerate}
\end{proof}

\lowerbound*
\begin{proof}
    Consider at time $0$ the client $f$ sends a list of requests which cannot fit in the memory at once.
    Because of work-conserving, client $f$ will fill the whole running batch.
    In this case, client $f$ is backlogged, and any new query is not processed until the existing queries finish processing.  Assume that all existing queries finish at time $T$, and that at time $\epsilon$ with $\epsilon$ close to 0, a second client $g$ sends another batch of requests. Now during the time interval $[\epsilon, T]$, both clients $f, g$ are backlogged since there exist  queries from both clients in the queue. At time $T$, client $f$ received service from the first batch of processing, which can be up to $w_q \cdot M$ if the memory is luckily fully utilized. Thus we have
    \[ W_f(\epsilon, T) = w_q \cdot M.\]
    On the other hand, client $g$ did not receive any service during the time period $[\epsilon, T]$. Thus $W_g(\epsilon, T) = 0$. In this case, we have constructed an instance with 
    \[ \vert W_f(t_1, t_2) - W_g(t_1, t_2) \vert \geq w_q \cdot M. \]
\end{proof}

\simpleiso*
\begin{proof}
    Let the counter for $f$ be $c_f$ after line 13 for $r_f$. Before $D_{r_f}$, since $r_f$ is always in the queue, the counter for f will not be lifted. Since there is no running batch of $f$ in the server, line 21 will select $r_f$ to be the next one for $f$. 
    Lemma~\ref{lem:invariant} shows that for any other client g,
    \[c_g - c_f < \max (w_p \cdot L_{input}, w_q\cdot M).\]
    In the worst case where these counters are incremented sequentially, it will take at most $2*(n-1)*\frac{\max(w_p \cdot L_{input}, w_q\cdot M)}{a}$. Thus, giving a bound for the dispatch time of $r_f$.
\end{proof}

\nopunish*
\begin{proof}

If $g$ is not backlogged during the entire $[t_1, t_2)$, then $W_g(t_1, t_2) \leq U$, the theorem trivially holds.
Next, assume $g$ is backlogged at some point during $[t_1, t_2)$.
Let $t_1'$, $t_2'$ be the first time and the last time g is backlogged between $[t_1, t_2)$.
Since there is no request submitted in $[t_1, t_1')$ and $[t_2', t_2)$, we have
\begin{equation}
\label{eq:iso_end_bound}
W_g(t_1, t_1') \leq  U, \quad W_g(t_2', t_2) \leq  U.
\end{equation}
Since $c_i$'s in \Cref{alg:vtc} are non-decreasing,
\begin{equation}
\label{eq:iso_increase_1}
c_f^{(t_1)} \leq c_f^{(t_1')} \leq c_f^{(t_2')} \leq c_f^{(t_2)},
\end{equation}
\begin{equation}
\label{eq:iso_increase_2}
c_g^{(t_1)} \leq c_g^{(t_1')} \leq c_g^{(t_2')} \leq c_g^{(t_2)}.
\end{equation}
According to Lemma~\ref{lem:invariant} there is,
\[c_g^{(t_2')} \leq c_f^{(t_2')} + U, \quad c_g^{(t_1')} \geq c_f^{(t_1')} - U.\]
By \Cref{eq:iso_increase_1} and \Cref{eq:iso_increase_2}:
\[c_g^{(t_2')} \leq c_f^{(t_2)} + U, \quad c_g^{(t_1')} \geq c_f^{(t_1)} - U.\]
Since $W_g(t_1', t_2') \leq c_g^{(t_2')} - c_g^{(t_1')}$, there is
\[ W_g(t_1', t_2') \leq c_f^{(t_2)} - c_f^{(t_1)} + 2U.\]
Combine with \Cref{eq:iso_end_bound}, there is:
\begin{align*}    
W_g(t_1, t_2) &= W_g(t_1, t_1') + W_g(t_1', t_2') + W_g(t_2', t_2)\\
&\leq c_f^{(t_2)} - c_f^{(t_1)} + 4U.
\end{align*}
Since f is backlogged during $(t_1, t_2)$, \[W_f(t_1, t_2) = c_f^{(t_2)} - c_f^{(t_1)}\] Thus: \[W_f(t_1, t_2) \geq W_g(t_1, t_2) - 4U\]

\end{proof}

\isolation*
\begin{proof}
We prove by contradiction. Assume there is a request from $f$ that has not been dispatched in $t_2$, i.e., $f$ is backlogged at $t_2$. Since $f$ is not backlogged at $t_1$, there exists a (non-empty) set of time steps such that $f$ becomes backlogged. We let $t$ be the largest element in the set, i.e. $f$ is backlogged at any time in $[t, t_2)$.
We claim that $W_f(t, t_2) \geq \frac{T(t, t_2)}{n(t, t_2)} - 4U$. 

From the pigeonhole principle,
there is at least one client $g$ who has received services $W_g(t, t_2)\geq \frac{T(t, t_2)}{n(t, t_2)}$.
If $f=g$, the claim holds.
If not, from \Cref{thm:nopunish}, we have
\[ W_f(t, t_2) \geq W_g(t, t_2) - 4U \geq \frac{T(t, t_2)}{n(t, t_2)} - 4U. \]

Since $f$ swicthes from non-backlogged to backlogged at $t$, requests sent before $t$ at most contributes a $U$ increase in $W_f(t, t_2)$. Thus, requests sent in $(t, t_2)$ at least contribute to $\frac{T(t, t_2)}{n} - 5U$, which contradicts to the assumption in the theorem.
\end{proof}

\section{Advanced VTC Variants}

This section presents additional experiments on various variants of VTC.
\Cref{sec:weighted_vtc} evaluate weighted VTC, which is introduced in \Cref{sec:method_weighted_vtc}.
In \Cref{sec:vtc_profile}, we show concretely how VTC can be tailored to specific cost functions, using a profiled service cost function as an example.
In \Cref{sec:vtc_length}, we include more analysis of VTC with length prediction, which is introduced in \Cref{sec:method_vtc_length}.  We empirically show its effectiveness in obtaining a better service discrepancy.

For all experiments shown in this section, we run Llama-2-7b on A10G (24GB), using the memory pool of 10000 tokens for KV cache.

\subsection{VTC for Weighted Fairness}
\label{sec:weighted_vtc}

\Cref{fig:weighted_vtc} demonstrates the effectiveness of the weighted VTC in managing clients with varied priority levels. We conducted a test using a synthetic workload involving four overloaded clients. The results depicted in \Cref{fig:priority_vtc} were achieved using standard VTC, which illustrates the comparable levels of service received by all four clients. In contrast, \Cref{fig:priority_wvtc}, which was obtained through the application of weighted VTC, shows differentiated service levels. The clients were assigned weights of 1, 2, 3, and 4, respectively, and the resulting service distribution closely adhered to these ratios.

\begin{figure}[ht]
    \subfloat[VTC]{\includegraphics[width=0.23\textwidth]{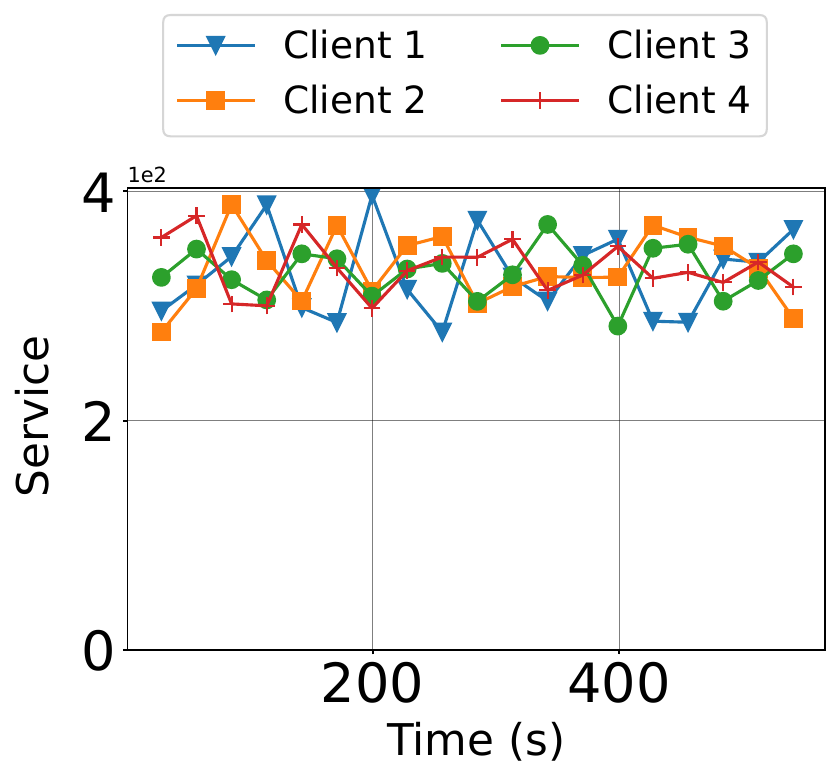}
    \label{fig:priority_vtc}}
    \hfill
    \subfloat[Weighted VTC]{\includegraphics[width=0.23\textwidth]{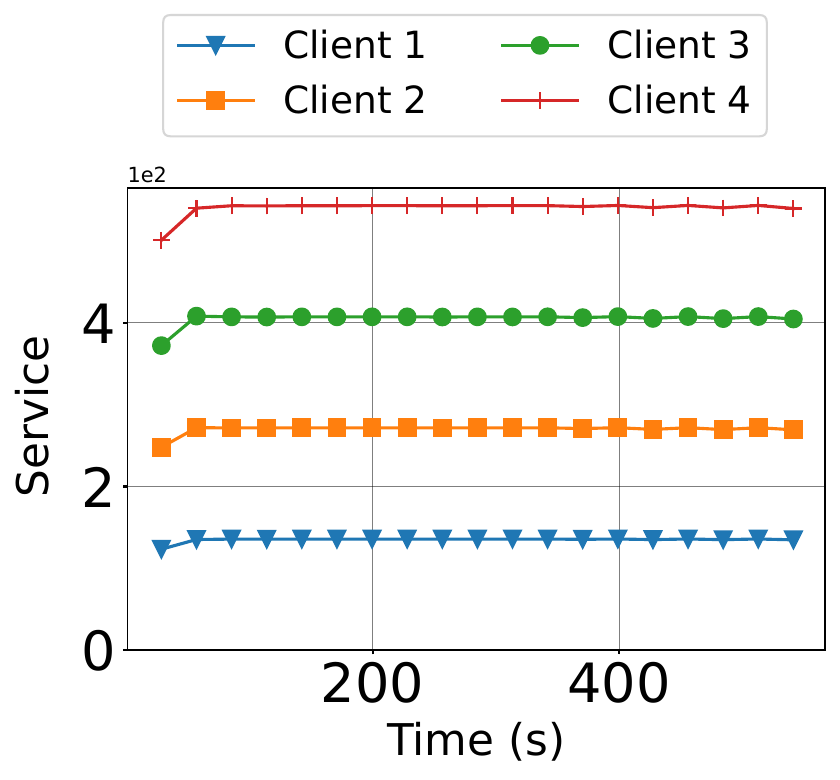}
    \label{fig:priority_wvtc}}
    \caption{Received service during the 10 minutes of the synthetic overloaded workload with input and output length both at 256. The figure on the left is obtained through standard VTC. The figure on the right is obtained through weighted VTC with weights 1:2:3:4 for the 4 clients.}
    \label{fig:weighted_vtc}
\end{figure}

\subsection{VTC with Profiled Cost Function}
\label{sec:vtc_profile}

In this section, we demonstrate the generalizability of the token cost function used in VTC (see \Cref{sec:general_cost}) by using a profiled service cost function.

To match our experimental setup, we profiled the inference time for Llama-2-7b on an A10G (24GB) across various conditions, as shown in Figure \ref{fig:cost_func}. We employed a batch size that utilizes the entire memory pool for each data point corresponding to specific input and output lengths. Consequently, shorter lengths allow for larger batch sizes, while longer lengths necessitate smaller ones. The prefill time is determined by dividing the total prefill time of the batch by the batch size. Similarly, the decode time is calculated by dividing the time taken to decode all tokens in the batch by the batch size. The function $h(n_p, n_q)$ is defined as the sum of prefill and decode times for the data point with input length $n_p$ and output length $n_q$.

\begin{figure}[ht]
    \subfloat[Prefill Time]{\includegraphics[width=0.23\textwidth]{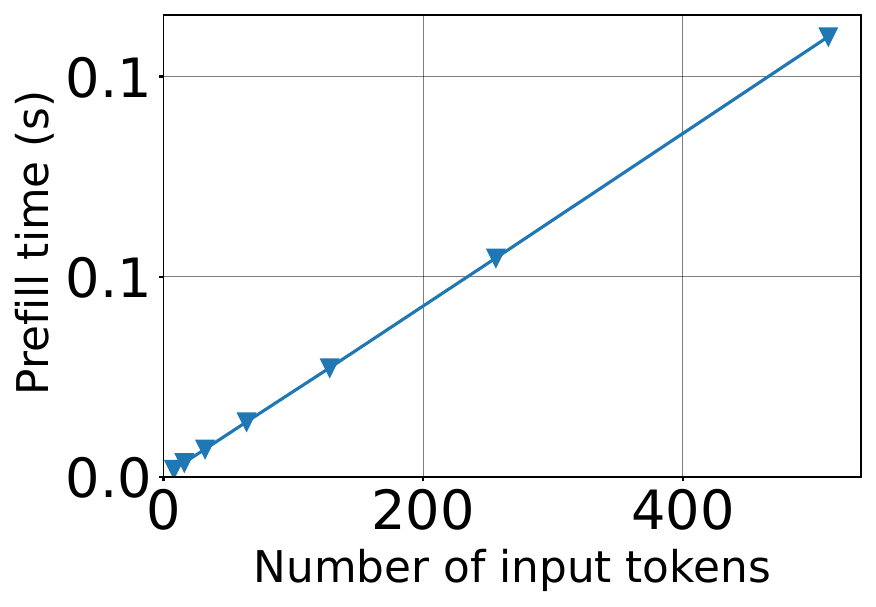}}
    \hfill
    \subfloat[Decode Time]{\includegraphics[width=0.23\textwidth]{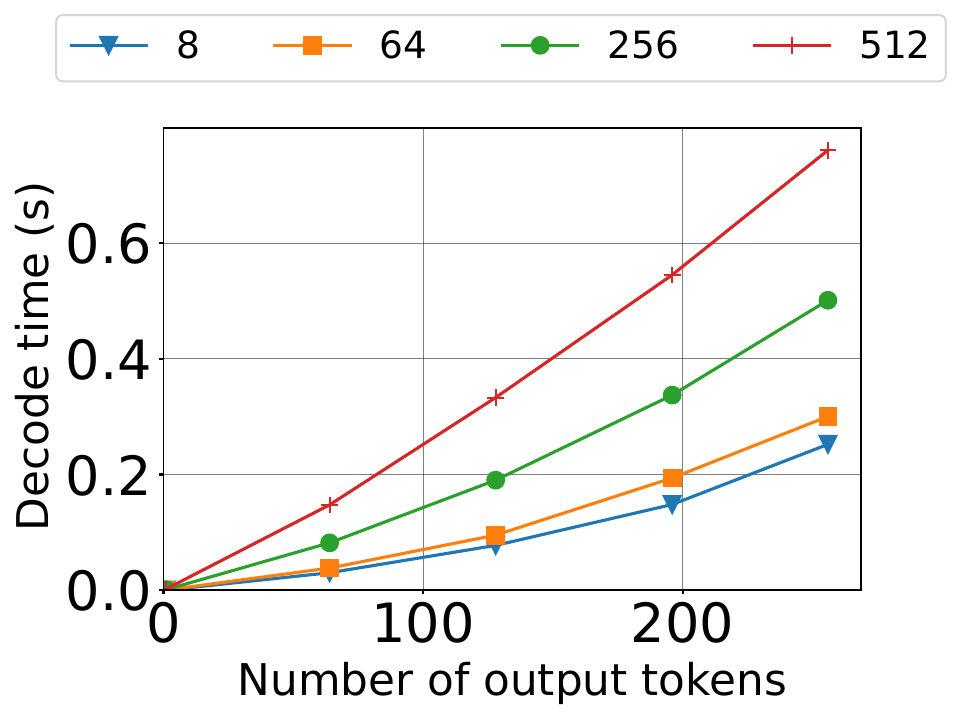}}
    \caption{Profiled prefill and decode time in different settings. For each data point, the batch size is set to the maximum to fulfill the memory pool (full utilization). The prefill time and decode time are all divided by the batch size. For the figure on the right, the legend for each curve denotes its number of input tokens.}
    \label{fig:cost_func}
\end{figure}

When considering the same total number of input and output tokens, the decode time for scenarios involving all output tokens is about 2 to 5 times the prefill time for scenarios involving all input tokens. The profiled cost function does not follow a linear model. We proceeded to fit the profiled data points and adjusted the coefficients to derive the following cost function:
\[ h(n_p, n_q) = 2.1 \cdot n_p + n_q + 0.04 \cdot n_p n_q + 0.032 \cdot n_q^2 + 11.46 \]

We conducted real trace experiments using this profiled cost function as the metric, the results of which are presented in \Cref{tab:quant_fairness_profile}.
The disparity between VTC and other baseline methods is insignificant because clients with low request rates, when starved, do not substantially impact the overall service difference. However, as observed in \Cref{fig:profile_h_response_time}, VTC successfully maintains low response times for clients with low request rates, a feat not matched by other baselines except for LCF.
In the case of LCF, clients with consistently high request rates face undue penalties, resulting in excessively high response times.
We reinforce our findings by assessing the profiled cost function on a synthetically overloaded workload to highlight the differences between VTC and FCFS, as shown in \Cref{tab:profile_2_clients}.

The results empirically show that VTC can achieve fairness using a customized cost function. However, our goal is not to determine the optimal cost function or pricing model, as these can vary based on numerous factors in a production environment and may change over time. The investigation into the cost function and pricing model is designated for future research.

\begin{table}[ht]
\centering
\resizebox{1\columnwidth}{!}{
\begin{tabular}{c|ccccc}
\toprule
Scheduler & Max Diff & Avg Diff & Diff Var & Throu & Isolation \\ 
\midrule 
FCFS & 743.23 & 457.29 & 26645.42 & 777 & No\\ 
LCF & 709.35 & 384.78 & 23299.20 & 778 & Some\\
\midrule
VTC & 707.35 & 368.74 & 21918.67 & 780 & Yes\\ 
VTC(predict) & 617.22 & 337.05 & 11803.41 & 778 & Yes\\
\textbf{VTC(oracle)} & \textbf{387.43} & \textbf{277.18} & \textbf{4541.57} & \textbf{783} & \textbf{Yes}\\
\midrule
RPM(5) & 230.78 & 151.00 & 823.15 & 340 & Some\\ 
RPM(20) & 445.34 & 270.51 & 5938.52 & 694 & Some\\ 
RPM(30) & 801.16 & 377.22 & 25980.39 & 747 & Some\\ 
\bottomrule
\end{tabular}
}
\caption{Results run on real workload under the profiled cost function introduced in \Cref{sec:vtc_profile}. The service difference is counted by summing the service difference between each client and the client who received the maximum services.
Throughput is the total number of tokens (including input and output tokens) processed divided by the total execution time.}
\label{tab:quant_fairness_profile}
\end{table}

\begin{table}[ht]
\centering
\resizebox{1\columnwidth}{!}{
\begin{tabular}{c|cccc}
\toprule
Scheduler & Max Diff & Avg Diff & Diff Var & Throughput \\ 
\midrule
FCFS & 323.18 & 317.13 & 15.98 & 876\\ 
VTC & 137.27 & 74.87 & 2819.40 & 900\\ 
VTC(oracle) & 4.28 & 0.34 & 0.91 & 893\\ 
\bottomrule
\end{tabular}
}
\caption{Results run on the synthetic overloaded workload with 2 clients under the profiled cost function introduced in \Cref{sec:vtc_profile}. The work difference is counted by summing the work difference between each client and the client who received the maximum services.}
\label{tab:profile_2_clients}
\end{table}

\begin{figure}[ht]
    \centering
    \subfloat[VTC (oracle)]{ \includegraphics[width=0.23\textwidth]{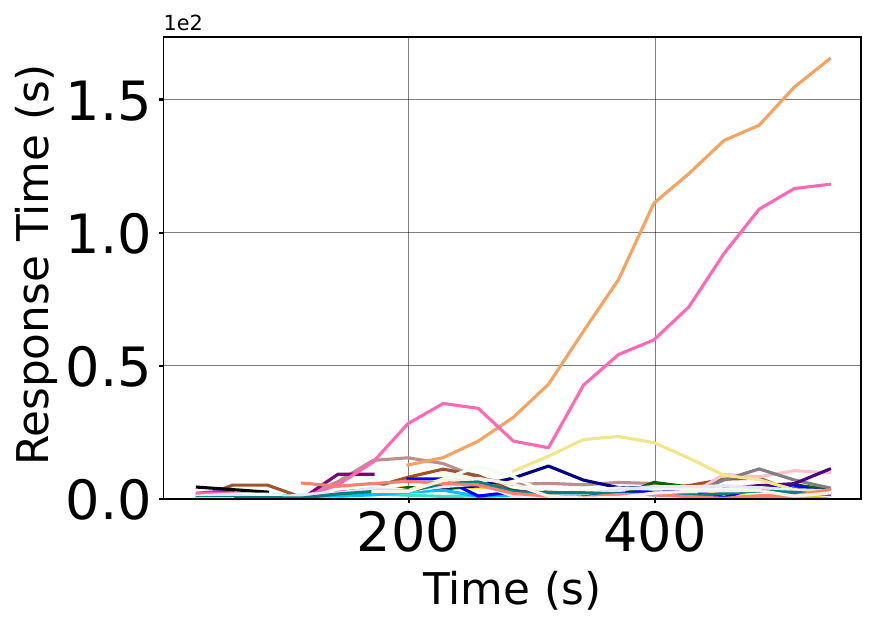}}
    \hfill
    \subfloat[VTC]{\includegraphics[width=0.23\textwidth]{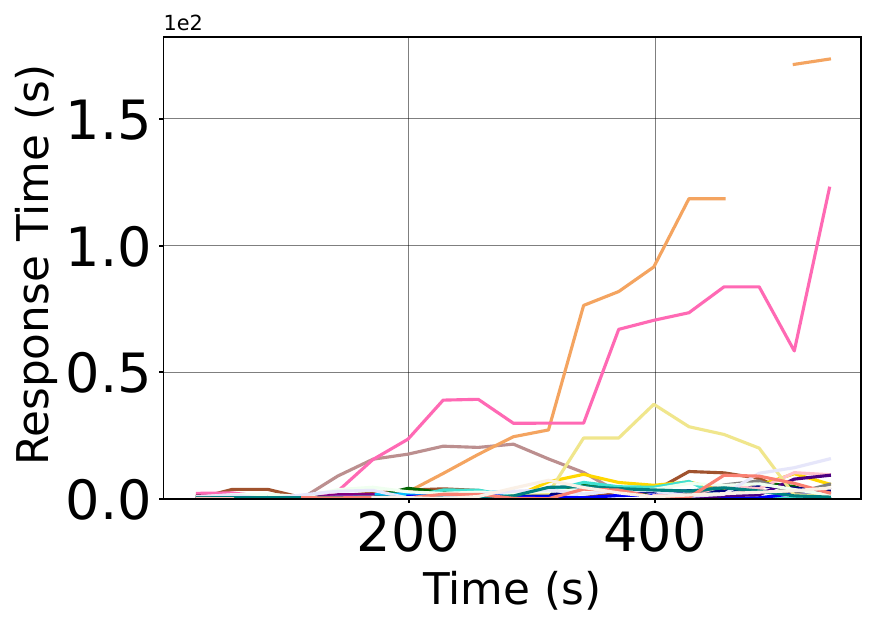}}
    
    \hfill
    
    \subfloat[RPM(20)]{    \includegraphics[width=0.23\textwidth]{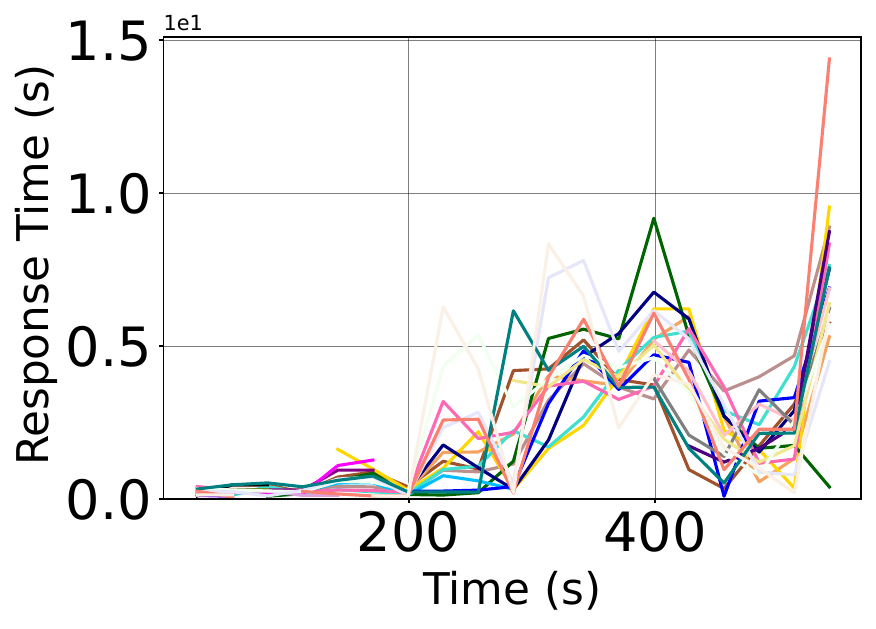}}
    \hfill
    \subfloat[RPM(30)]{\includegraphics[width=0.23\textwidth]{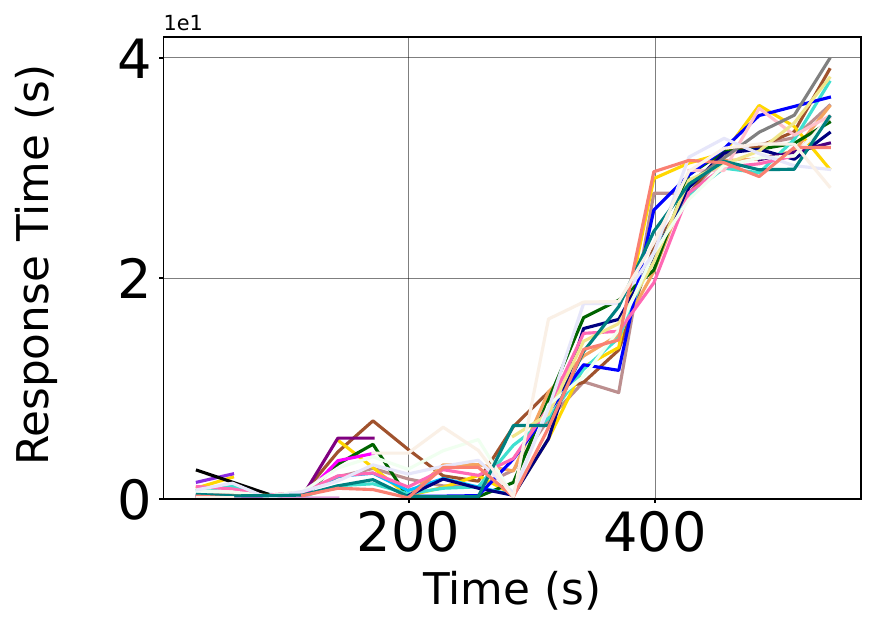}}
    
    \hfill
    
    \subfloat[FCFS]{    \includegraphics[width=0.23\textwidth]{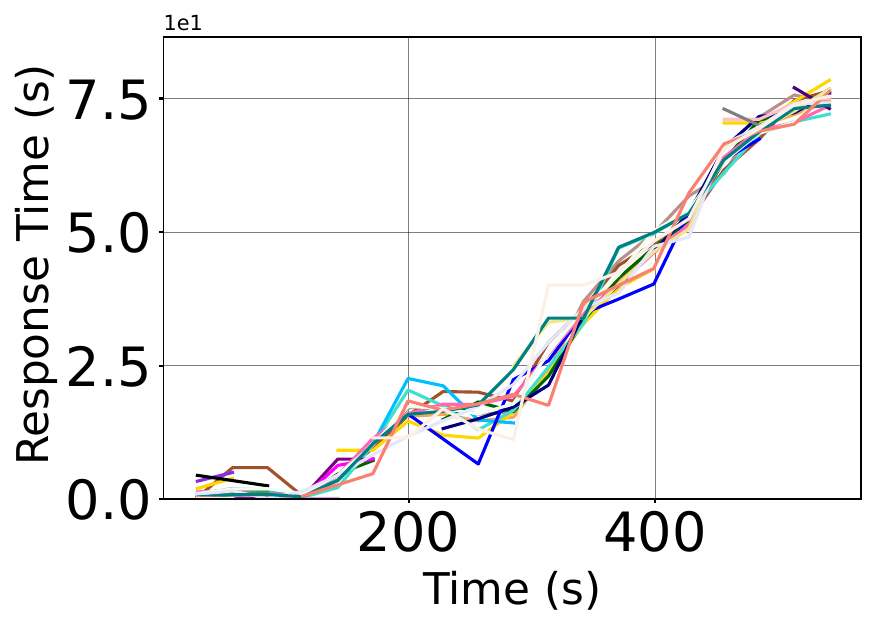}}
    \hfill
    \subfloat[LCF]{    \includegraphics[width=0.23\textwidth]{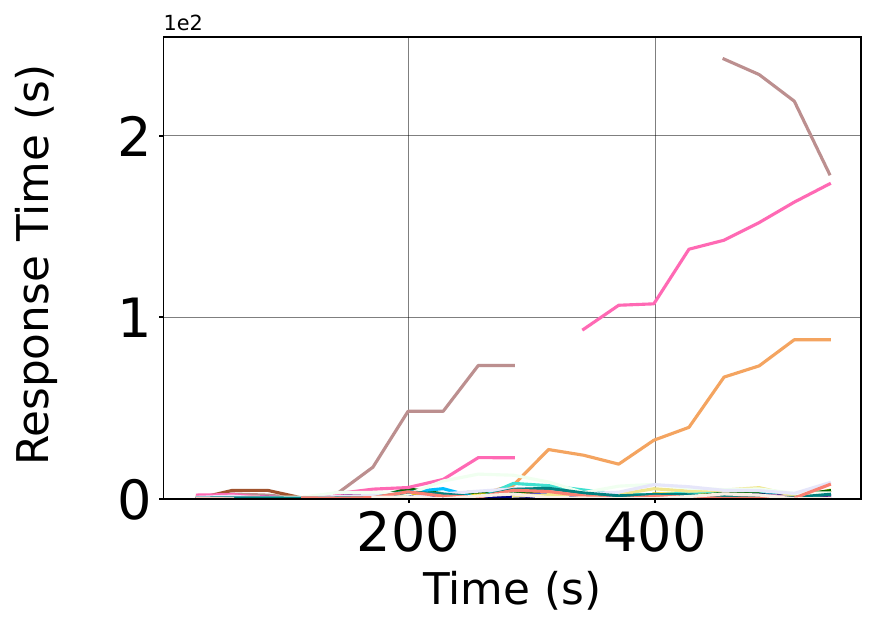}}
    \caption{Response time of the 27 clients during the 10 minutes of real trace simulation using different schedulers. The VTC style schedulers are using the profiled cost function introduced in \Cref{sec:vtc_profile}. There are some curves that show disconnected because, during some periods, a client may have no requests served. Requests distribution see \Cref{fig:real_req_rate}.}
    \label{fig:profile_h_response_time}
\end{figure}

\subsection{VTC with Length Prediction}
\label{sec:vtc_length}

The adapted pseudocode for VTC with length prediction is detailed in Algorithm \ref{alg:vtc_length_predict}. In line 25, the cost associated with the predicted number of output tokens is preemptively calculated. Lines 32-37 describe the adjustments made to the cost to correspond with the actual number of output tokens produced.

\Cref{fig:vtc_len_overload} demonstrates how length prediction reduces service discrepancies among clients in a synthetic workload scenario where all clients are overloaded.
"VTC (oracle)" refers to a simulation using a predictor with 100\% accuracy. "VTC ($\pm 50\%$)" simulates a predictor that randomly selects a value within 50\% of the actual output length, either above or below. 
While standard VTC ensures that the absolute differences in services received by clients remain bounded and do not grow over time, VTC with length prediction significantly lowers these differences throughout the test period, even with a prediction error margin of 50\%.
\Cref{tab:pred_len_2_clients} and \Cref{tab:pred_len_n_clients} provide quantitative assessments of the service discrepancies among overloaded clients under the same conditions.

\begin{figure}[ht]
    \subfloat[2 Clients]{\includegraphics[width=0.23\textwidth]{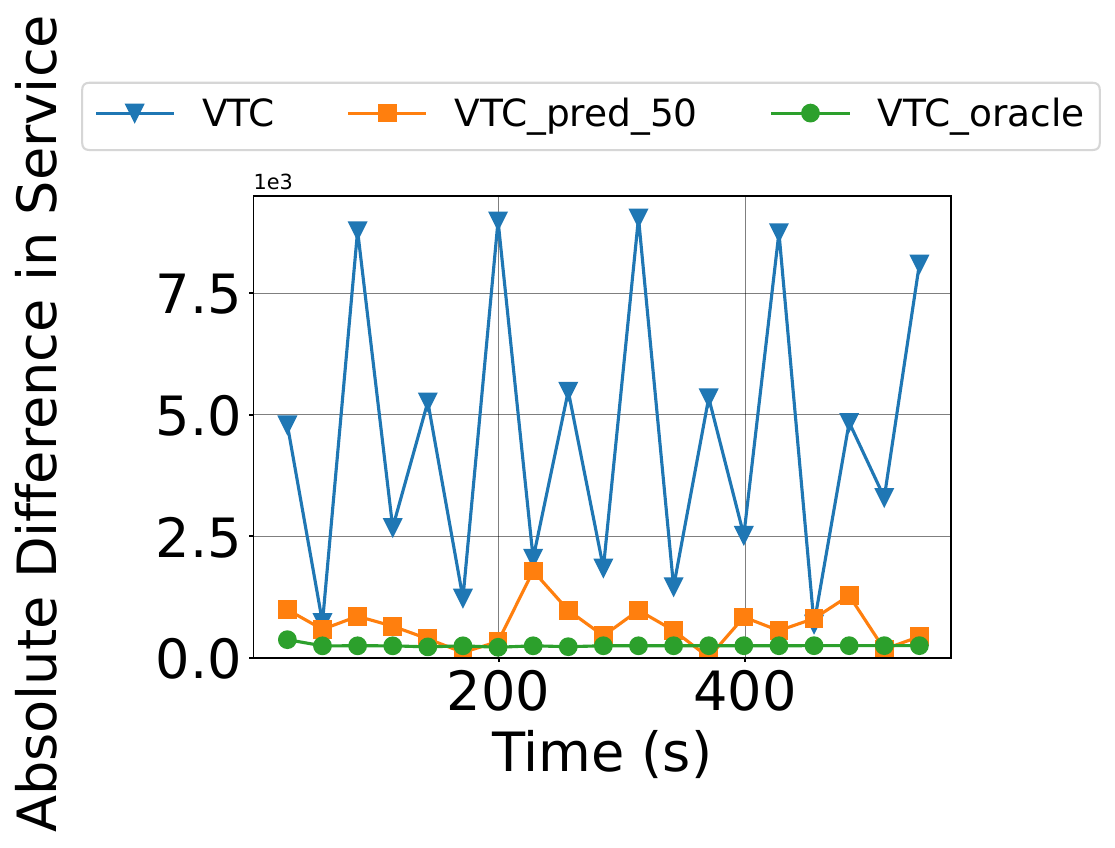}}
    \hfill
    \subfloat[8 Clients]{\includegraphics[width=0.23\textwidth]{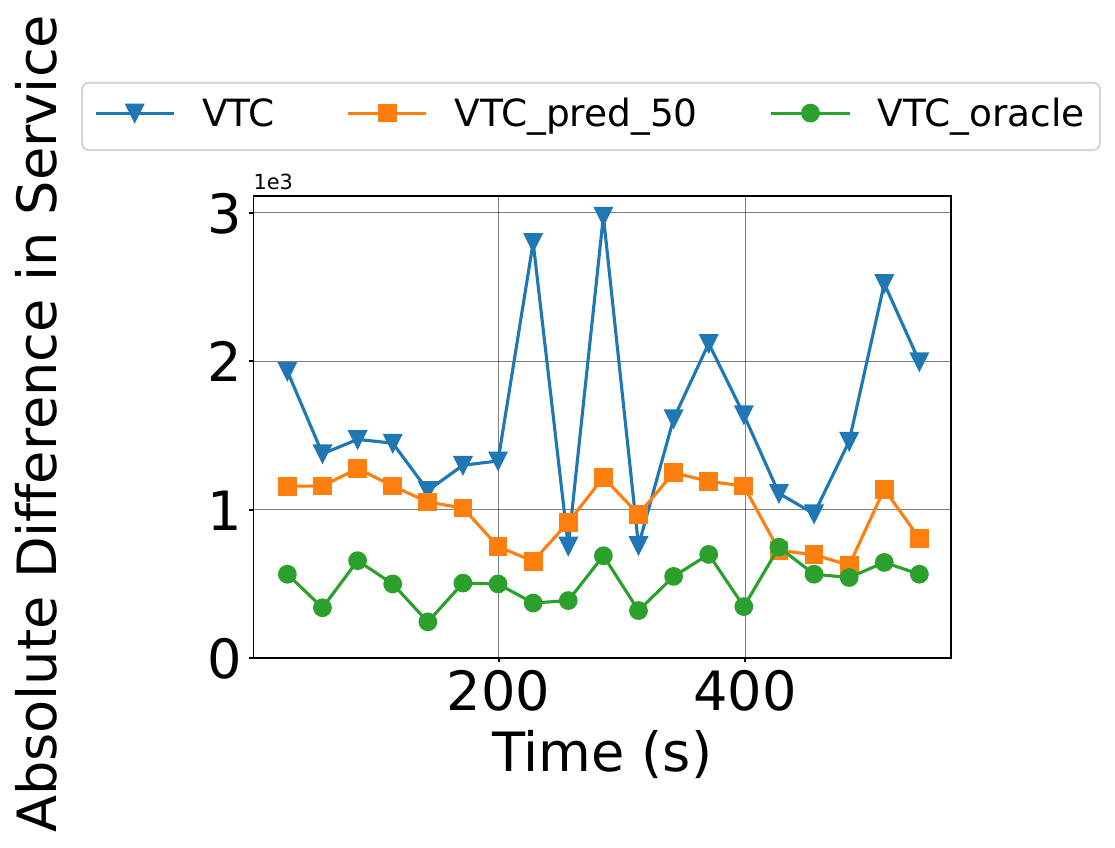}}
    \caption{The figures illustrate the maximum difference in accumulated service received by clients during a 10-minute period of synthetic workload, where both the input and output lengths are set at 256. The left figure is derived from a simulation involving two clients, while the right figure comes from a simulation involving eight clients. In both scenarios, the request rate for each client surpasses the available capacity, resulting in continuous backlogging of each client.}
    \label{fig:vtc_len_overload}
\end{figure}

\begin{table}[ht]
\centering
\resizebox{1\columnwidth}{!}{
\begin{tabular}{c|cccc}
\toprule
Scheduler & Max Diff & Avg Diff & Diff Var & Throughput \\ 
\midrule
VTC & 192.88 & 103.77 & 6981.24 & 893\\
VTC ($\pm 50\%$) & 33.98 & 12.54 & 111.94 & 904\\ 
VTC (oracle) & 5.87 & 0.51 & 1.71 & 895\\ 
\bottomrule
\end{tabular}
}
\caption{Results run on 10-minute synthetic workload same with \Cref{fig:vtc_len_overload} for 2 clients. The service difference is counted by summing the work difference between each client and the client who received the maximum services.
Throughput is the total number of tokens (including input and output tokens) processed divided by the total execution time.}
\label{tab:pred_len_2_clients}
\end{table}

\begin{table}[ht]
\centering
\resizebox{1\columnwidth}{!}{
\begin{tabular}{c|cccc}
\toprule
Scheduler & Max Diff & Avg Diff & Diff Var & Throughput \\ 
\midrule
VTC & 322.16 & 162.20 & 5151.49 & 875\\ 
VTC ($\pm$ 50\%) & 99.43 & 66.32 & 487.10 & 875\\ 
VTC (oracle) & 43.23 & 36.34 & 56.52 & 875\\
\bottomrule
\end{tabular}
}
\caption{Results run on 10-minute synthetic workload same with \Cref{fig:vtc_len_overload} for 8 clients. The service difference is counted by summing the work difference between each client and the client who received the maximum services.
Throughput is the total number of tokens (including input and output tokens) processed divided by the total execution time.}
\label{tab:pred_len_n_clients}
\end{table}

\begin{algorithm}[t!]
\caption{VTC with Length Prediction}
\begin{algorithmic}[1]
\Require request trace, input token weight $w_p$, output token weight $w_q$, upper bound from \Cref{eq:invariant} denoted as $U$.
\State let current batch $B \leftarrow \emptyset$
\State let $c_i \leftarrow 0$ for all client $i$
\State let $Q$ denote the waiting queue, which is dynamically changing.
\State $\triangleright$ \texttt{with monitoring stream:}
\While{True}
    \If {new request $r$ from client $u$ arrived}
        \If {not $\exists r' \in Q, client(r')=u$}
            \If {$Q=\emptyset$}
                \State let $l\leftarrow$ the last client left $Q$
                \State $c_u \leftarrow \max\{c_u, c_{l}\}$
            \Else
                \State $P \leftarrow \{i \mid \exists r' \in Q, client(r')=i\}$
                \State $c_u \leftarrow \max\{c_u, \min\{c_i \mid i\in P\}\}$
            \EndIf
        \EndIf
        \State $Q \leftarrow Q + r$
    \EndIf
\EndWhile
\State $\triangleright$ \texttt{with execution stream:}
\While{True}
    \If{can\_add\_new\_request()}
        \State $B_{new} \leftarrow \emptyset$
        \While{True} 
            \State let $k \leftarrow \argmin_{i \in \{client(r) \mid r\in Q\}} c_i$
            \State let $r$ be the earliest request in $Q$ from $k$.
            \If{$r$ cannot fit in the memory}
                \State Break
            \EndIf
            \State $c_k \leftarrow c_k + w_p\cdot input\_length(r)$
            \State $c_k \leftarrow c_k + w_q\cdot predicted\_output\_length(r)$
            \State $B_{new} \leftarrow B_{new} + r$
            \State $Q \leftarrow Q - r$
        \EndWhile
        \State forward\_prefill($B_{new}$)
        \State $B \leftarrow B + B_{new}$
    \EndIf
    \State forward\_decode($B$)
    \State $\triangleright$ \texttt{Adjust the cost of output tokens}
    \For{each $r \in B$}
        \State $\delta \leftarrow output\_len(r) - predicted\_output\_len(r)$
        \If{$\delta > 0$}
            \State $c_{client(r)} \leftarrow c_{client(r)} + w_q$
        \EndIf
        \If{$r$ is finished and $\delta < 0$}
            \State $c_{client(r)} \leftarrow c_{client(r)} + w_q\cdot \delta$
        \EndIf
    \EndFor
    \State $B \leftarrow$ filter\_finished\_requests($B$)
\EndWhile
\end{algorithmic}
\label{alg:vtc_length_predict}
\end{algorithm}

\begin{figure}[t!]
    \centering
    \begin{subfigure}[b]{0.22\textwidth}
    \includegraphics[width=\textwidth]{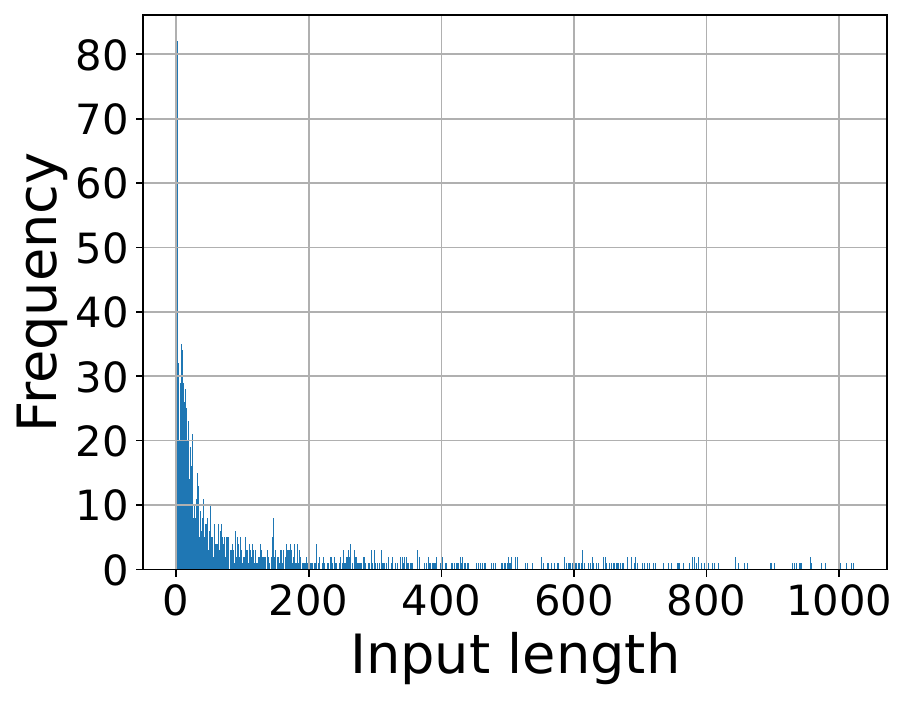}
    \end{subfigure}
    \hfill
    \begin{subfigure}[b]{0.235\textwidth}
    \includegraphics[width=\textwidth]{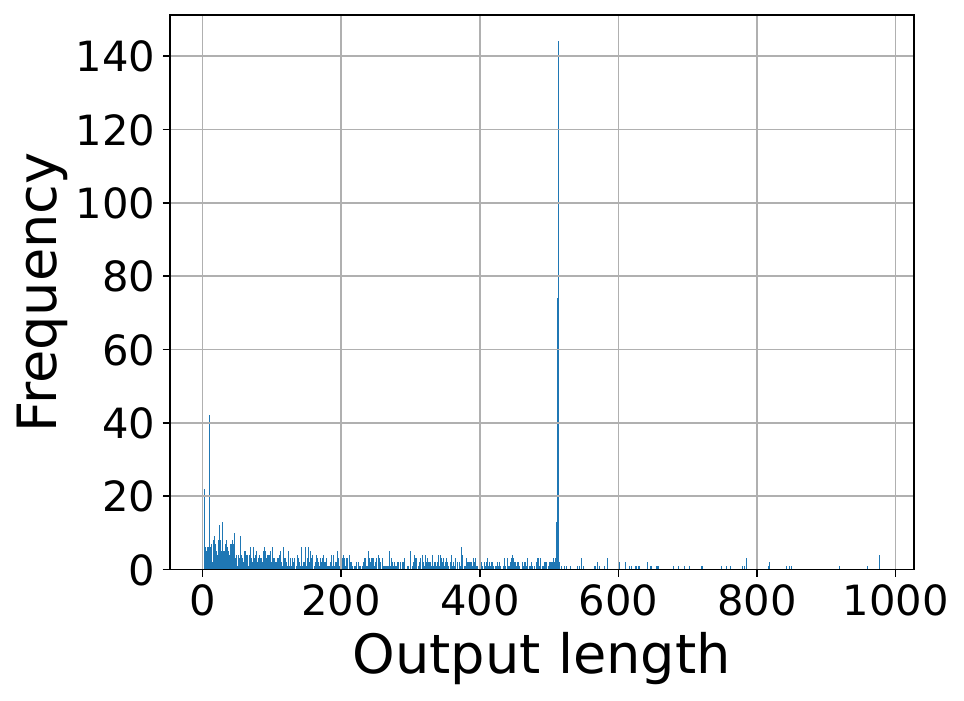}
    \end{subfigure}
    \caption{Request input and output length distribution in the real workload trace during the sampled 10 minutes duration with re-scale. The average input length is 136, and the average output length is 256. The input and output lengths have the range of $[2, 1021]$ and $[2, 977]$, respectively.}
    \label{fig:real_req_len}
\end{figure}

\section{Discussions}
\label{sec:discussions}

\subsection{VTC Integration in Real Systems}
\label{sec:vtc_integration}

In \Cref{alg:vtc}, we have shown an example of VTC integration with continuous batching.
In implementation, VTC integration should be a simple change in the request scheduler.
Generally, for an existing serving system, there are three modules that need to be modified.
First, the monitoring stream handles counter-lifting when a new request comes, as shown in \Cref{alg:vtc_general}, which is the same as in \Cref{alg:vtc}.
Second, when new tokens have been processed, the counters should be updated according to a pre-defined cost function as discussed in \Cref{sec:general_cost}.
Third, when new requests need to be selected for processing, we schedule the request from a user with the lowest counter first.
The added modules are demonstrated in \Cref{alg:vtc_general}.
We are assuming a customized cost function $h(n_p, n_q)$ as introduced in \Cref{sec:measurement}.
At line 22,
$n_p^r, n_q^r$ denote the number of processed input and output tokens, and
$n_p^{r(old)}, n_q^{r(old)}$ denote the number of processed input and output tokens before processing the new tokens.

Those modules for maintaining the virtual token counters and selecting requests according to the counters could be additive features of an existing serving system.
However, in some cases, VTC is possibly in conflict with a scheduling algorithm that optimizes performance while being against fairness. Cache-aware scheduling introduced in \cite{zheng2023efficiently} is an example in which requests with shared prefixes will always be prioritized.
A natural solution to combine the two is adding a policy of switching between the two schedulers by setting tolerable fairness bounds.
We leave such exploration as future research.

\begin{algorithm}[ht]
\caption{General VTC}
\begin{algorithmic}[1]
\label{alg:vtc_general}
\Require request trace, input token weight $w_p$, output token weight $w_q$, upper bound from \Cref{eq:invariant} denoted as $U$.
\State let current batch $B \leftarrow \emptyset$
\State let $c_i \leftarrow 0$ for all client $i$
\State let $Q$ denote the waiting queue, which is dynamically changing.
\State $\triangleright$ \texttt{with monitoring stream:}
\While{True}
    \If {new request $r$ from client $u$ arrived}
        \If {not $\exists r' \in Q, client(r')=u$}
            \If {$Q=\emptyset$}
                \State let $l\leftarrow$ the last client left $Q$
                \State $c_u \leftarrow \max\{c_u, c_{l}\}$
            \Else
                \State $P \leftarrow \{i \mid \exists r' \in Q, client(r')=i\}$
                \State $c_u \leftarrow \max\{c_u, \min\{c_i \mid i\in P\}\}$
            \EndIf
        \EndIf
        \State $Q \leftarrow Q + r$
    \EndIf
\EndWhile
\State $\triangleright$ \texttt{when process new request:}
\If {add\_new\_request()}
    \State let $k \leftarrow \argmin_{i \in \{client(r) \mid r\in Q\}} c_i$
    \State let $r$ be the earliest request in $Q$ from $k$.
    \State $Q \leftarrow Q - r$
    \State original process when selecting $r$.
\EndIf
\State $\triangleright$ \texttt{when new tokens been processed:}
\State $c_i \leftarrow c_i + \sum_{r\mid client(r) = i} \left( h(n_p^r, n_q^r) - h(n_p^{r(old)}, n_q^{r(old)}) \right)$
\end{algorithmic}
\end{algorithm}

\subsection{Adapted Deficit Round Robin}
\label{sec:drr}

We have briefly discussed in \Cref{sec:challenge} why Deficit Round Robin (DRR) cannot be directly applied.
In this section, we discuss an adaptation of Deficit Round Robin~\cite{drr} and show it is equivalent to our proposed VTC scheduler.

The original DRR can be described as follows:
\begin{itemize}
\item[1.] The algorithm maintains a constant $Q$, which is the quantum that each client has.
\item[2.] Every client maintains a variable $C_i$ that represents its deficit, which is initialized as 0.
\item[3.] On each round, the algorithm visits each client with a non-empty queue and schedules its requests as many as possible if the incurred cost $P$ is less than or equal to $Q + C_i$.
The $C_i$ is then updated to $Q + C_i - P$ if $P > C_i$ or $C_i - P$ if else.
\end{itemize}

The obstacle to applying DRR in LLM serving is that we do not know how many requests we should schedule to meet the requirement of $P \leq Q+C_i$ since the number of output tokens is unknown in advance.

We then give an adapted version for LLM serving:
\begin{itemize}
\item[1.] The algorithm maintains a constant $Q$, which is still the quantum that each client has.

\item[2.] Every client maintains a variable $C_i$ that represents its debt, which is initialized as 0.

\item[3.] In each round, the algorithm processes each client. If $C_i \leq 0$, it refills $C_i$ by adding $Q$ to it. Should the updated $C_i$ become positive, the algorithm schedules as many requests as possible, such that the cost associated with the prompt tokens $P$ slightly exceeds $C_i$ with the addition of the last scheduled request. After scheduling, $P$ is subtracted from $C_i$.

\item[4.] Each time a new token is decoded, the associated cost is deducted from the respective $C_i$. Consequently, $C_i$ may become negative, exceeding the value of $Q$ multiple times, and it might require waiting through several rounds before it can be scheduled again.
\end{itemize}

Fairness is no longer strictly bounded by $Q$, yet a smaller $Q$ promotes a tighter constraint.
When $Q = \epsilon$ is extremely small, smaller than the cost of a single prompt token, the algorithms revert to functioning like the VTC algorithm.
This is because each round results in one of two outcomes: either all $C_i$ values remain non-positive, prompting another round, or the highest $C_i$ turns positive and the corresponding client is scheduled. 
The client with the highest $C_i$ is the one who has received the least service, which corresponds to having the smallest virtual counter in VTC.

If a client has no requests in the queue at a given time, it will cease to be refilled once $C_i \geq 0$.
When a new request arrives, its $C_i$ will be within $(0, \epsilon]$, approximating $\max_i{C_i}$.
The $\max_i{C_i}$ remains within the range of $(0, \epsilon]$ because the algorithm persistently adds $\epsilon$ to $C_i$ to maintain it positive, but then rapidly pulls it back into the negative by scheduling new requests. This process mirrors the counter lift mechanism in VTC.

In addition to its similarity to VTC, practically, simulating repeated round-robin with a small quantum $Q$ is inefficient. Therefore, we focus solely on analyzing VTC in this paper, leaving the discussion of the round-robin simulation here for reference.

\subsection{Future Work}
\label{sec:future_work}

\paragraph{Preemption}
\label{sec:preemption}

As we mentioned in \Cref{sec:prelim_llm_serving}, this paper focuses on how to integrate
fair scheduling with continuous batching, and leaving an
investigation on preemption as an orthogonal future work.
But we still would like to discuss how preemption will affect the VTC algorithm, and point out a possible future research on it.

The nature of unpredictable length in a no-preemption framework directly affects the fairness bound in the main theorem \Cref{thm:main_fairness}, which is $U=2\max(w_p\cdot L_{input}, w_q\cdot M)$.
Intuitively, the worst case occurs when many requests from one client are added, generating a large number of tokens that cannot be preempted.
During the process, other clients cannot catch up arbitrarily because the memory is occupied.
Essentially, this is caused by an underestimation of a future number of tokens, similarly explained in the ablation study (\Cref{sec:ablation}) and VTC with length prediction (\Cref{sec:method_vtc_length}).

In \Cref{remark:restrict_mem}, we mentioned that we could restrict the memory usage for each client in the running batch to obtain a better bound. However, this can potentially lower the overall throughput because the memory may not always be fully utilized.
Having a preemption mechanism could be a good solution to address the problem of underestimating and tightening the bound.
Basically, if the difference in service is larger than a threshold, we can preempt the requests in processing and swap in requests from clients with lower counters.

\paragraph{VTC for distributed systems}
Integrating VTC in a distributed LLM serving system is an interesting direction for future work. For a distributed setup where there are many replicas of serving engines, we will have a central request dispatcher where we can keep the token counter and enforce the algorithm (this is similar to hierarchical fair sharing~\cite{bennett1997hierarchical} in the network domain, and multi-queue fair queuing~\cite{hedayati2019multi}). The bound now is dependent on the total memory capacity of all the serving engines. However, in the distributed setting, the counter will be updated by different serving engines concurrently, raising the problem of counter synchronization, which will be interesting to explore as a future work.

\paragraph{VTC and Auto-scaling}
The VTC algorithm does not rely on a constant capacity. Adding and removing GPUs will not affect the algorithm but may need a hierarchical virtual counter as discussed in the paragraph about distributed systems.
However, auto-scaling is a possible approach to mitigate the issue of throughput degradation in RPM.
The resources can be auto-scaled to fit the fluctuating traffic, but this requires flexible and responsive resource management.
Auto-scaling has its own challenges, including operational cost overhead, inaccurate workload prediction, and delays.
A combination of VTC and auto-scaling is a future direction worth exploring.

\end{document}